\documentclass{article}

\usepackage[preprint]{neurips_2025}

\usepackage[utf8]{inputenc} 
\usepackage[T1]{fontenc}    
\usepackage{hyperref}       
\usepackage{url}            
\usepackage{booktabs}       
\usepackage{amsfonts}       
\usepackage{nicefrac}       
\usepackage{microtype}      
\usepackage{xcolor}         
\usepackage{xspace}
\usepackage{enumitem} 
\usepackage{listings}
\usepackage{amsmath}
\usepackage{amssymb}
\usepackage{mathtools}
\usepackage{amsthm}
\usepackage{xspace}
\usepackage[]{xcolor}
\usepackage{thm-restate}
\usepackage{subcaption}

\usepackage{algorithm}
\usepackage[noend]{algpseudocode}
\usepackage{wrapfig}
\theoremstyle{plain}
\newtheorem{theorem}{Theorem}[section]

\theoremstyle{definition}
\newtheorem{definition}[theorem]{Definition}

\theoremstyle{remark}

\newcommand{\Tool}{\textsc{DINGO}\xspace}
\newcommand{\unconstrained}{Unconstrained\xspace}
\newcommand{\greedyconstrained}{Greedy Constrained\xspace}
\newcommand{\bestofbaseline}{Best of Greedy + Unconstrained \xspace}
\newcommand{\upto}{68\%}

\newcommand{\pad}[1]{P_{#1}}
\newcommand{\concat}{\cdot}
\newcommand{\vect}[1]{\pmb{#1}}

\newcommand{\transform}[1]{\mathcal{N}_{#1}}
\newcommand{\maskPred}[1]{\mathcal{M}_{#1}}

\newcommand{\lang}[1]{L(#1)}

\newcommand{\diffusion}[1]{f_{#1}}

\newcommand{\blocklen}{d}
\newcommand{\probstr}[1]{\mathbb{R}_{+}^{|\inalphaLLM{}|\times #1}}
\newcommand{\regex}{\mathcal{R}}
\newcommand{\pre}[1]{L_{P}(#1)}
\newcommand{\nomask}{(\inalphaLLM{}\setminus \bot)}

\newcommand{\llm}[1]{\mathcal{L}_{#1}}

\newcommand{\inalphaLLM}[1]{V^{#1}}

\newcommand{\decode}[1]{D_{m,#1}}
\newcommand{\constrained}[1]{D_{m,#1, \regex}}
\newcommand{\sub}[1]{\mathcal{S(#1)}}
\newcommand{\distri}[1]{\mathcal{D}_{#1}}
\newcommand{\opt}[1]{#1^{*}}
\newcommand{\dpstate}[2]{W[#1, #2]}
\newcommand{\dppar}[2]{Pr[#1, #2]}

\newcommand{\dfa}[1]{D_{#1}}
\newcommand{\dfastates}{Q}
\newcommand{\alphabets}{\Sigma}
\newcommand{\transitions}{\delta}
\newcommand{\tokentrans}[1]{\delta_{#1}}

\newcommand{\dfastart}{q_0}
\newcommand{\dfafinal}{F}
\newcommand{\dfadef}{(\dfastates, \alphabets, \transitions, \dfastart, \dfafinal)}
\newcommand{\dfadefr}[1]{(\dfastates, \alphabets, \tokentrans{\regex}, \dfastart, \dfafinal)}

\newcommand{\expstate}{2^{\dfastates}}

\newcommand{\live}[1]{#1_{l}}
\newcommand{\cost}[3]{V_{#3}(#1, #2)}
\newcommand{\costS}[1]{V_{#1}}
\newcommand{\tranS}[1]{T_{#1}}

\DeclareMathOperator*{\argmax}{arg\,max}

\title{\Tool: Constrained Inference for Diffusion LLMs}

\author{Tarun Suresh$^*$, Debangshu Banerjee\thanks{Equal contributing authors ordered randomly}, 
Shubham Ugare, Sasa Misailovic, Gagandeep Singh
 \\
Department of Computer Science, \\
University of Illinois Urbana-Champaign}

\begin{document}
\maketitle

\begin{abstract}
Diffusion LLMs have emerged as a promising alternative to conventional autoregressive LLMs, offering substantial potential for improving runtime efficiency. However, existing diffusion models fail to provably enforce user-specified formal constraints, such as regular expressions, which makes them unreliable for tasks that require structured outputs, such as fixed-schema JSON generation. Unlike autoregressive models, which generate tokens sequentially, diffusion LLMs predict a block of tokens in parallel. This parallelism makes traditional constrained decoding algorithms, designed to enforce constraints with sequential token prediction, ineffective at preserving the true output distribution. To address this limitation, we propose \Tool, a dynamic programming-based constrained decoding strategy that is both efficient and provably distribution-preserving. \Tool enables sampling of output strings with the highest probability under the model’s predicted distribution while strictly adhering to any user-specified regular expression. On standard symbolic math and JSON generation benchmarks, \Tool achieves up to a $\upto$ points of improvement over unconstrained inference.
\end{abstract}
\section{Introduction}
Autoregressive LLMs demonstrate impressive performance across a wide range of tasks, including logical reasoning~\citep{pan2023logiclmempoweringlargelanguage}, theorem proving~\citep{yang2023leandojotheoremprovingretrievalaugmented}, and code generation~\citep{chen2021evaluatinglargelanguagemodels}. 
However, because they generate one token at a time, they can be slow when producing long responses. 
Recent work has explored using diffusion models to accelerate token generation by predicting blocks of tokens in parallel. 
For tasks such as logical reasoning, where the LLM output is fed into symbolic solvers like Z3~\citep{fedoseev2024llm}, syntactic correctness of the output is essential. 
Prior works~\citep{poesia2022synchromesh, ugare2024improving, loula2025syntactic} have shown that LLMs frequently make syntactic and semantic errors, often generating structurally invalid outputs that cause downstream tasks to fail due to unparsable input. 
To mitigate this issue, constrained decoding has emerged as a promising approach that provably ensures structural correctness by projecting the LLM output onto a set of valid strings, typically defined by a regular grammar or, more generally, a context-free grammar (CFG). However, existing constrained decoding techniques are designed specifically for autoregressive LLMs and rely on their step-by-step generation process to prune invalid tokens that cannot lead to structurally valid outputs. At each generation step, the decoder selects the highest-probability token from the set of valid options, based on the LLM's output distribution.

In contrast, diffusion LLMs predict blocks of tokens in parallel without sequential dependencies, making existing constrained decoding algorithms incompatible. Furthermore, greedy token selection in autoregressive models maximizes the probability locally at each step but can be suboptimal over an entire sequence, potentially leading to structurally valid yet lower-quality outputs that fail to maximize the overall probability of valid strings. \cite{lew2023sequential, park2024grammaraligned} have reported this distortion in output distribution for autoregressive LLMs under constrained decoding. Therefore, any constrained decoding algorithm for diffusion LLMs should also ensure that enforcing formal constraints does not come at the cost of distorting the true output distribution. 

\textbf{Key Challenges:} Diffusion LLMs generate a block of tokens starting from a fully masked string composed of special mask tokens $\bot$, and iteratively unmask one or more tokens at each step until producing a fully unmasked output. Each unmasking step (referred to as a diffusion step) can unmask tokens at arbitrary positions in the block, with no left-to-right sequential dependency across steps. As a result, designing constrained decoding for diffusion LLMs requires addressing the following:

\begin{itemize}[leftmargin=*]
\item \textbf{RQ1:} Efficiently detecting invalid tokens and restricting token choices at each diffusion step to ensure the final unmasked string is always structurally correct.
\item \textbf{RQ2:} Ensuring the generated token block maximizes the probability under the output distribution.
\end{itemize}

\textbf{Contributions:} We present the first constrained decoding algorithm for diffusion LLMs, making the following contributions:

\begin{itemize}[leftmargin=*]
\item We introduce \Tool, the first constrained decoding algorithm for diffusion LLMs that supports any user-specified regular expression. \Tool provably ensures that the output string is always a valid prefix of some string in the target regular language.

\item \Tool uses dynamic programming to ensure that the output string achieves the maximum probability among all valid strings over the output block with respect to the true output distribution. This approach guarantees scalability while maintaining optimality (e.g., maximizing the probability), in contrast to existing methods such as~\citep{park2024grammaraligned}, which rely on repeated resampling. Resampling-based methods are computationally expensive and unsuitable for practical deployment.
\item Extensive experiments on multiple open-source diffusion LLMs and benchmarks show that \Tool significantly outperforms standard unconstrained decoding, achieving up to a $\upto$ improvement on challenging tasks such as the GSM-symbolic benchmark for symbolic reasoning~\citep{mirzadeh2024gsmsymbolicunderstandinglimitationsmathematical} and a JSON generation benchmark~\citep{jsoneval}.
\end{itemize}

\noindent \textbf{Roadmap: }We provide the necessary background in Section~\ref{sec:background}, formalize constrained decoding for diffusion LLMs in Section~\ref{sec:probForm}, describe the \Tool algorithm along with its correctness and optimality proofs in Section~\ref{sec:algoDetail}, and present experimental results in Section~\ref{sec:experiments}.

\section{Background}
\label{sec:background}
\textbf{Notation: }: In the rest of the paper, we use small case letters $x$ for constants, bold small case letters ($\vect{x}$) for strings, capital letters $X$ for functions, $\concat$ for string concatenation, $|\vect{x}|$ to denote the length of $\vect{x}$.


\noindent \textbf{Diffusion LLM:} The diffusion LLM $\llm{m,n} : \inalphaLLM{m} \to \inalphaLLM{n}$ processes finite strings $\vect{p} \in \inalphaLLM{m}$ over a finite alphabet $\inalphaLLM{}$ including the special mask symbol $\bot$ and produces the output string $\vect{o} \in \inalphaLLM{n}$. Typically $\vect{o} = \vect{p}  \concat \vect{r}$ with length $n$ represents the entire output string of $\llm{}$ where $\vect{p}$ is the input prompt, $\vect{r}$ is the response, and $m + |\vect{r}| = n$. $\llm{}$ can compute the response $\vect{r}$ over a single block~\cite{austin2021structured, dream2025, llada} in pure diffusion setup or over multiple blocks i.e. $\vect{r_1} \concat \vect{r_2}\cdots\vect{r_k}$ in a semi-autoregressive setup where different blocks are computed sequentially from left to right \citep{han2023ssdlmsemiautoregressivesimplexbaseddiffusion, arriola2025block}. 



                

At a high level, to compute a block of tokens of size $\blocklen$, $\llm{}$ pads the prompt $\vect{p}$ with a fully masked suffix, resulting in $\vect{p} \concat \bot^{\blocklen}$, where $\bot^{\blocklen}$ denotes a sequence of $\blocklen$ special mask tokens $\bot$. The model then iteratively unmasks a subset of these tokens at each step, ultimately producing a fully unmasked output string $\vect{o}$. Each such step is referred to as a diffusion step, and $\llm{}$ typically applies $T$ diffusion steps to compute $\vect{o}$. The number of steps $T$ is usually a fixed, predetermined constant satisfying $T < \blocklen$, which enables greater scalability compared to their autoregressive counterparts.

\begin{definition}[Diffusion step] 
A diffusion step $\diffusion{n} : \inalphaLLM{n} \times \mathbb{N} \to \inalphaLLM{n}$ applies a single unmasking step to a masked (or, a partially masked) string of length to compute a new masked (or, possibly unmasked) string of the same length. The first argument represents the input string appended with the output block while the second argument dictates the number of masked tokens in the output string.
\end{definition}
Each diffusion step $\diffusion{n}$ consists of two components: a transformer step $\transform{n} : \inalphaLLM{n} \to \probstr{n}$, which predicts the token probability distribution at each output position, and a mask prediction step $\maskPred{n} : \probstr{n} \times \mathbb{N} \to \probstr{n}$, which determines which token positions to remask. Typically, for each position, the mask prediction step identifies the token with the highest probability and compares these maximum probabilities across positions. $\maskPred{n}$ then greedily remasks positions with relatively lower max-probability scores \citep{llada} and produces the modified token distribution. Further details about $\transform{n}$ and $\maskPred{n}$ are in Appendix~\ref{appen:tansformAndRemask}.

\noindent Formally, the diffusion step is defined as $\diffusion{n}(\vect{x}_{i-1}, i) = \decode{n}(\maskPred{n}(\transform{n}(\vect{x}_{i-1}), i))$ where $\decode{n} : \probstr{n} \to \inalphaLLM{n}$ is the decoder. We now use the diffusion step to formally define the diffusion LLM for generating strings of length $n$ in either a single-block or multi-block setting.

\begin{definition}[Single block diffusion LLM]
A diffusion LLM that outputs a block of $\blocklen$ tokens given an input $\vect{p} \in \inalphaLLM{m}$ using $T$ diffusion steps is a function $\llm{m,n} : \inalphaLLM{m} \to \inalphaLLM{n}$, where $n = m + \blocklen$, and the output is $\vect{o} = \vect{p} \concat \vect{r} = \llm{m,n}(\vect{p})$. Let $\diffusion{n} : \inalphaLLM{n} \times \mathbb{N} \to \inalphaLLM{n}$ denote a single diffusion step, and let $\pad{m,n} : \inalphaLLM{m} \to \inalphaLLM{n}$ be the padding function. Then the output is computed as $\vect{o} = \llm{m,n}(\vect{p}) = \vect{x}_T$, where:
$\vect{x}_0 = \pad{m,n}(\vect{p}) = \vect{p} \concat \bot^{\blocklen}$ and $\vect{x}_i = \diffusion{n}(\vect{x}_{i-1}, i) \text{ for } 1 \leq i \leq T$.
\end{definition}
\begin{definition}[Semi Autoregressive diffusion LLM] 
In the semi-autoregressive setup, given an input $\vect{p} \in \inalphaLLM{m}$, the output $\vect{o} \in \inalphaLLM{m + \blocklen \times k}$ is generated over $k$ blocks, where each block is computed via a call to the single block diffusion model. The output of the $i$-th diffusion model call is $\vect{x}_i = \llm{m_i, n_i}(\vect{x}_{i-1})$, with $\vect{x}_0 = \vect{p}$ and the final output $\vect{o} = \vect{x}_k$. The input and output lengths for each block are defined as $m_i = m + (i - 1) \times \blocklen$ and $n_i = m + i \times \blocklen$ for all $1 \leq i \leq k$.
\end{definition}

\noindent \textbf{DFA and regular expression: } We provide necessary definitions regarding regular expression.
\begin{definition}(DFA)
A DFA $\dfa{\regex} = \dfadef$ for a regular expression $\regex$ is a finite-state machine that deterministically processes input strings to decide membership in the language $\lang{\regex} \subseteq \alphabets^{*}$ defined by $\regex$. It consists of states $\dfastates$, a start state $\dfastart$, a set of accepting states $\dfafinal$, and transition rules $\delta : \dfastates \times \alphabets \to \dfastates$ and the input alphabet $\alphabets$.
\end{definition}
\begin{definition}[extended transition function]
\label{def:extended}
The extended transition function $\transitions^* : \alphabets^{*} \times \dfastates \to \dfastates$ maps an input $(\vect{w}, q)$ to the resulting state $q_r$, obtained by sequentially applying $\tokentrans{}$ to each character $c_i$ in $\vect{w} = c_1 \cdots c_{|\vect{w}|}$, starting from state $q$.
\end{definition}
\begin{definition}[Live DFA states]
\label{def:liveState}
Given a DFA $\dfadef$, let $\live{Q}$ represent the set of live states such that $q \in \live{Q}$ iff $\exists w \in \alphabets^*$ s.t.  $\transitions^*(\vect{w}, q) \in \dfafinal$. 
\end{definition}

\section{Optimal Constrained Decoding}
\label{sec:probForm}
We formalize the correctness and optimality of constrained decoding for any diffusion LLM with respect to a user-defined regular expression $\regex$. 
Given $\regex$, let $\lang{\regex} \subseteq \alphabets^{*} \subseteq (V \setminus \bot)^{*}$ denote the set of all finite strings that satisfy the expression $\regex$.

\noindent \textbf{Correctness: }A valid constrained decoding algorithm must ensure that the output string always remains a valid \textit{prefix} of some string in $\lang{\regex}$, effectively eliminating any output that cannot be extended into valid completions. By treating the output string as a prefix rather than a fully completed string, we can accommodate the semi-autoregressive setup, where blocks of tokens are appended to the right of the current output. This approach avoids prematurely rejecting strings that may lead to valid completions in subsequent blocks and also aligns with the notion of correctness adopted in existing constrained decoding algorithms for the autoregressive LLM \citep{ugare2024syncodellmgenerationgrammar, banerjee2025crane}. We denote the set of all valid prefixes of $\lang{\regex}$ as $\pre{\regex}$.

\noindent Each diffusion step $\diffusion{n}$ produces a string over the vocabulary $V$, which may include one or more special mask tokens $\bot$. These tokens act as placeholders for actual (non-mask) tokens that will be filled in during future diffusion steps. To account for these future substitutions, we define a masked (or partially masked) string as valid if there exists a replacement for all mask tokens such that the resulting fully unmasked string is a valid prefix of some string in $\lang{\regex}$. To formalize this notion, we first define the \textit{substitution set}, which represents the set of fully unmasked strings obtained by replacing all mask tokens in a masked or partially masked string. We then use substitution sets to define the correctness of the constrained decoder.
\begin{definition}[Substitution Set]
\label{def:substitutionset}
Given a masked (or, partially masked) string $\vect{x} \in \inalphaLLM{n}$, the \textit{substitution set} $\sub{\vect{x}} \subseteq (V \setminus \{\bot\})^n$ is the set of all fully unmasked strings obtained by replacing each occurrence of $\bot$ in $\vect{x}$ with a token from $V \setminus \{\bot\}$. For unmasked strings with no $\bot$, $\sub{\vect{x}} = \{\vect{x}\}$ 
\end{definition}
\begin{definition}[Correctness of Constrained decoder]
\label{def:correctness}
Any deterministic  decoder $\constrained{n} : \probstr{n} \to \inalphaLLM{n}$ is a valid constrained decoder if, for all $n \in \mathbb{N}$, input prompt $\vect{p}$ and for any output distribution $\distri{n}$ provided as $n$ probability vectors each of size $|\inalphaLLM{}|$, there exists an unmasked string $\vect{x}$ in the substitution set $\mathcal{S}(\constrained{n}(\distri{n}))$ of the decoded output such that actual response $\vect{p}\concat\vect{r} = \vect{x}$ is a valid prefix i.e., $\vect{r} \in \pre{\regex}$. \footnote{More precisely, if there exists at least one $\vect{r}$ that is a valid prefix (i.e., $\vect{r} \in \pre{\regex}$), then constrained decoding is always capable of retrieving one of them. }
\end{definition}

\noindent \textbf{Optimality: }
Given a distribution $\distri{n}$ and a regular expression $\regex$, the set of decodings that are valid prefixes for $\regex$ (as defined in Definition~\ref{def:correctness}) may not be unique. 
An optimal constrained decoder selects, among all valid strings, the string that maximizes the probability under $\distri{n}$. 
The output distribution $\distri{n}$ is represented as $n$ vectors $\vect{v}_1, \dots, \vect{v}_n$, each of size $|\inalphaLLM{}|$, where the $i$-th vector $\vect{v}_i$ captures the token distribution at position $i$. 
For any masked position $j$, $\vect{v}_j$ assigns probability 1 to the mask token $\bot$ and 0 to all other tokens. 
Assuming the input prompt has length $m$, the token distribution of the actual response is given by $\vect{v}_{m+1}, \dots, \vect{v}_n$. For any output string $\vect{r} = t_{m+1} \cdots t_n$, let $P(\vect{r} \mid \vect{v}_{m+1} \dots \vect{v}_{n})$ denote the probability of the string $\vect{r}$ under the output distribution. Then, the optimal constrained decoding can be formalized as follows:
\begin{align}
   \opt{\vect{r}} =  \argmax_{\vect{r}} P(\vect{r} \;|\; \vect{v}_{m+1}\dots\vect{v}_{n})\text{  s.t. $\exists \vect{x} \in V^*. (\vect{x} \in \sub{\vect{r}}) \wedge (\vect{x} \in \pre{\regex})$} \label{eq:probform}
\end{align}
\noindent Since the token distributions $\vect{v}_{m+1}, \dots, \vect{v}_n$ are independent across positions, the probability of the string $\vect{r}$ can be written as $P(\vect{r}\;|\;\vect{v}_{m+1}\dots\vect{v}_n) = \prod_{i=m+1}^{n} \vect{v}_{i}[t_i]$ where $\vect{v}_i[t_i]$ denotes the probability assigned to token $t_i$ by the vector $\vect{v}_i$. Using this, we can rewrite the optimization problem from Eq.~\ref{eq:probform} as follows: 
\begin{align}
\opt{\vect{r}} =  \argmax_{\vect{r} = t_{m+1} \cdots t_n} \prod_{i=m+1}^{n} \vect{v}_i[t_i]\text{  s.t. $\exists \vect{x} \in V^*. (\vect{x} \in \sub{\vect{r}}) \wedge (\vect{x} \in \pre{\regex})$} \label{eq:probForm2}
\end{align}


\section{\Tool Algorithm}
\label{sec:algoDetail}
The search space for Eq.~\ref{eq:probForm2} is exponential-- $|\inalphaLLM{}|^{\blocklen}$, where $\blocklen = n - m$ denotes the block length, making naive enumeration-based methods impractical.
To efficiently retrieve the optimal output string $\opt{\vect{r}}$ from Eq.~\ref{eq:probForm2}, \Tool leverages dynamic programming. Given a regular expression $\regex$, it first modifies the transition function to handle the mask symbol $\bot$, which is then utilized during inference.    

\subsection{Precomputation}
For a user-provided $\regex$ and the corresponding DFA $\dfa{\regex} = \dfadefr{\regex}$ (referred to as character-level DFA) with $\alphabets \subseteq \nomask$, we first construct the token-level DFA $\dfa{t} = (\dfastates, \nomask, \tokentrans{t}, \dfastart, \dfafinal)$ recognizing $\lang{\regex}$ over strings generated by $\llm{}$. A single token $\vect{t} \in \nomask$ can span across multiple characters in $\alphabets$ i.e. $\vect{t} = c_1\cdots c_{l}$ where $c_i \in \alphabets$. To construct the token-level transition function $\tokentrans{t} : \dfastates \times \nomask \to \dfastates$, we process each token $\vect{t} \in \nomask$ and state $q \in \dfastates$ by executing the character-level DFA $\dfa{\regex}$ on the sequence of constituent characters $c_1 \cdots c_l$, starting from state $q$, and record the resulting state $q_r$. We then define the token-level transition as $\tokentrans{t}(q, \vect{t}) = q_r$.

\noindent To handle the special mask token $\bot \in \inalphaLLM{}$, we define the transition function $\tokentrans{\bot} : \dfastates \to \expstate$. For each state $q \in \dfastates$, $\tokentrans{\bot}(q)$ returns the set of states $\dfastates_r \subseteq \dfastates$ that are reachable via a single token transition using $\tokentrans{t}$. Formally, $\tokentrans{\bot}(q) = \{q' \mid q' = \tokentrans{t}(q, \vect{t}); \vect{t} \in \nomask\}$. 
Since $\tokentrans{\bot}$ may return multiple states, it resembles the transition function of a non-deterministic finite automaton (NFA). 
The precomputation step combines $\tokentrans{t}$ and $\tokentrans{\bot}$ to define $\tokentrans{} : \dfastates \times \inalphaLLM{} \to \expstate$, which is used in the dynamic programming step. Using the token-level DFA $\dfa{t}$, we also construct the set of live states $\live{Q} \subseteq \dfastates$ (Definition~\ref{def:liveState}).
\begin{align*}
\tokentrans{}(q, \vect{t}) = 
\begin{cases}
    \{\tokentrans{t}(q, t)\} & \text{if } t \in \nomask, \\
    \tokentrans{\bot}(q) & \text{if } t = \bot.
\end{cases}
\end{align*}
\subsection{\Tool Dynamic Programming}
Before going into details, we present two key observations that lead to the decoding algorithm.

\noindent \textbf{Observation 1: }Determining whether a fully unmasked string $\vect{r} = t_1\cdots t_{\blocklen} \in \nomask^{*}$ is a valid prefix is equivalent to checking whether the resulting state $q_r$, obtained by applying $\tokentrans{}$ to the sequence $t_1\cdots t_{\blocklen}$ starting from $\dfastart$, is live. Similarly, for a partially (or fully) masked string $\vect{r}_{\bot}$, applying $\tokentrans{}$ to $t_1\cdots t_{\blocklen}$ yields a set of resulting states $\dfastates_r$. In this case, $\vect{r}_{\bot}$ is a valid prefix if and only if any state $q \in \dfastates_r$ is live (Definition~\ref{def:correctness}).

\noindent \textbf{Observation 2: }For optimality, it is sufficient to track the maximum probability path from the start state $q_0$ to each resulting state in $\dfastates_r$. Once these paths are computed, we select the one with the highest probability that leads to a live state. The corresponding string is the optimal string $\opt{\vect{r}}$ (or one of the optimal strings in case of multiple optimal solutions) for the optimization problem in Eq.~\ref{eq:probForm2}.

Based on these observations, the main challenge is to efficiently maintain the maximum probability path to each reachable state in $\dfastates_r$. We address this using a dynamic programming (DP) approach, similar to traditional graph-based DP algorithms such as~\citep{Forney:1973ly}.

\noindent \textbf{DP states: } For each token position $1 \leq i \leq \blocklen$ in the block, the DP maintains: a) $\dpstate{i}{q}$, which records the maximum probability with which a state $q \in \dfastates$ can be reached from the start state $q_0$ via transitions on some token sequence with length $i$; and b) $\dppar{q}{i}$, which stores the last transition, i.e., the previous state and the corresponding token, that led to the maximum probability stored in $\dpstate{i}{q}$. If a state $q$ is unreachable, then $\dpstate{i}{q} = 0$. Formally, given the probability vectors $\vect{v}_1, \dots, \vect{v}_i$, $\dpstate{i}{q}$ is defined as follows where $\tokentrans{t}^{*}$ is extended transition function (Definition~\ref{def:extended}). 
\begin{align*}
\dpstate{i}{q} = \max_{t_1 \dots t_i} \prod_{j=1}^{i} \vect{v}_j[t_j] \;\;\text{  s.t.  $q = \tokentrans{t}^{*}( t_{m+1}\cdots t_n, \dfastart)$}
\end{align*}
\noindent \textbf{DP state update: } Given the states at token position $i$, we describe the computation for position $i+1$. Initially, $\dpstate{i}{q} = 0$ for all $q \neq q_0$, and $\dpstate{i}{\dfastart} = 1$ (lines 1 -- 3 in Algo.~\ref{alg:dp_block}). To compute $\dpstate{i+1}{q}$ for each $q \in \dfastates$, we consider all tokens $t \in \inalphaLLM{}$ (including the mask token $\bot$) that can transition to $q$ from some previous state $q'$ at step $i$. Among all such transitions, we select the one with the highest probability and add it to the maximum probability path reaching $q'$ at step $i$. The value $\dppar{i+1}{q}$ stores the previous state and token that lead to the maximum probability path to $q$ at step $i+1$ (lines 12 -- 15 in Algo.~\ref{alg:dp_block}). Formally,
\begin{align*}
\cost{q}{q'}{i+1} = \begin{cases}
    \max\limits_{t \in \inalphaLLM{}}\; \vect{v}_{i+1}(t) \text{ s.t. $q \in \tokentrans{}(q', t)$} \\
    0  \text{  if $q, q'$ are not connected}   
\end{cases}& \dpstate{i+1}{q} = \max_{q'\in \dfastates} \dpstate{i}{q'} \times \cost{q}{q'}{i+1}
\end{align*}
\noindent \textbf{Path construction: }We consider all reachable states $q$ at the end of the block with $\dpstate{\blocklen}{q} > 0$. Among the live states $q_l \in \live{Q}$ satisfying this condition, we select the state $q_{\text{max}}$ with the highest value of $\dpstate{\blocklen}{q_l}$. We then use $Pr$ to iteratively reconstruct the token sequence backward that forms the maximum probability path starting from $q_{\text{max}}$ and ending at $\dfastart$ (lines 20 -- 22 in Algo.~\ref{alg:dp_block}).

\noindent \textbf{Semi-autoregressive setup: }In semi-autoregressive setup, we may not start from DFA start state $\dfastart$ since one or more blocks of tokens $\vect{r}_1\cdots\vect{r}_l$ may have been generated in left the current block. Provide the string $\vect{r}_1\cdots\vect{r}_l$ ends at a live state $q_{l}$, we can apply dynamic programming approach with the intializtion $\dpstate{0}{q_l} = 1$ and $\dpstate{0}{q} = 0$ for all state $q \neq q_l$. Details are in Appendix~\ref{appen:semiauto}.

\subsection{Correctness of \Tool}
\begin{restatable}{proposition}{correctnessThm}[Correctness] Given any regular expression $\regex$, input prompt $\vect{p}\in \inalphaLLM{m}$, block length $\blocklen$, output distribution $\distri{m+\blocklen} = \vect{v}_1\dots\vect{v}_{m+d}$, if $\pre{\regex} \cap \nomask^{d} \neq \{\}$ and $\vect{r} \sim \vect{v}_{m+1}\dots\vect{v}_{m+d}$ be the decoded string, then $\exists \vect{x} \in V^*. (\vect{x} \in \sub{\vect{r}}) \wedge (\vect{x} \in \pre{\regex})$ holds. 
\label{thm:correctness}
\end{restatable}
\textbf{Proof sketch: }\Tool ensures that if a state $q \in \dfastates$ is reachable in $i$ tokens, then $\dpstate{i}{q} > 0$ for all $1 \leq i \leq \blocklen$. Since $\pre{\regex} \cap \nomask^{d} \neq \{\}$, there exists a state $q_l \in \live{Q}$ that is reachable in $d$ steps. Therefore, $\dpstate{d}{q_{max}} > 0$ (see line 16 in Alg.\ref{alg:dp_block}). Consequently, there exists a sequence $\vect{x} \in \sub{\vect{r}}$ such that $\tokentrans{}^{*}(\vect{x}, \dfastart) = q_{max} \in \live{Q}$, implying that $\vect{x} \in \pre{\regex}$. Formal proof is in Appendix\ref{appen:proofs}. 


\begin{restatable}{proposition}{optimalityThm}[Optimality]
\label{thm:optimality}
Given any regular expression $\regex$, input prompt $\vect{p}\in \inalphaLLM{m}$, block length $\blocklen$, output distribution $\distri{m+\blocklen} = \vect{v}_1\dots\vect{v}_{m+d}$, if $\pre{\regex} \cap \nomask^{d} \neq \{\}$ and $\opt{\vect{r}} \sim \vect{v}_{m+1}\dots\vect{v}_{m+d}$ be the decoded string, then for any valid string $\vect{r}'$ satisfying $\exists \vect{x} \in V^*. (\vect{x} \in \sub{\vect{r'}}) \wedge (\vect{x} \in \pre{\regex})$, $P(\vect{r'} \;|\; \vect{v}_{m+1}\dots\vect{v}_{n}) \leq P(\opt{\vect{r}} \;|\; \vect{v}_{m+1}\dots\vect{v}_{n})$.
\end{restatable}
\textbf{Proof Sketch: } Formal proof is in Appendix~\ref{appen:proofs}.


\begin{algorithm}[t!]
    \small
    \caption{\Tool{} DP}
    \label{alg:dp_block}
    \begin{algorithmic}[1]
    \Require $\dfastart$, block length $\blocklen$, probability vectors $\vect{v}_1, \dots \vect{v}_{\blocklen}$ for the current block,  $\live{Q}$, $\dfastates$, $\tokentrans{}$.
    \State $\dpstate{0}{q} \gets 0$ for all $(q \in \dfastates) \wedge (q \neq q_0)$
    \State $\dpstate{0}{\dfastart} \gets 1$
    \State $\dppar{0}{q} \gets (\text{None}, \text{None})$  for all $(q \in \dfastates)$ \Comment{Initialization of the DP}
    \State $\costS{i} \gets \{\}$ for all $i \in \{1,\dots,\blocklen\}$\Comment{maximum token probability transtion $(q' \to  q)$ at position $i$}
    \State $\tranS{i} \gets \{\}$ for all $i \in \{1,\dots,\blocklen\}$ \Comment{token for the maximum probability transition $(q' \to q)$}
    \For{$i \in \{1,\dots, \blocklen\}$}
    \For{$(q \in \dfastates)$}
    \For{$t \in \inalphaLLM{}$}
    \State $q' \gets \delta(q, t)$
    \State $\costS{i}(q, q'), \tranS{i}(q, q') \gets $ MaxTransition$(\vect{v}_i, t, q, q')$   
    \EndFor
    \EndFor
    \EndFor
    \For{$i \in \{1,\dots, \blocklen\}$} \Comment{DP computation loop}
    \For{$(q \in \dfastates)\wedge (q' \in \dfastates)$}
        \If{$\dpstate{i}{q} < \dpstate{i-1}{q'}\times \costS{i}(q, q')$ }
        \State $\dpstate{i}{q} \gets \dpstate{i-1}{q'}\times \costS{i}(q, q')$ \Comment{Update maximum probability path to $q$}
        \State $\dppar{i}{q} \gets (q', \tranS{i}(q, q'))$  \Comment{Update the parents accordingly}
        \EndIf
    \EndFor
    \EndFor
    \State $q_{max} \gets \argmax_{q\in Q_l} \dpstate{\blocklen}{q}$
    \If{$\dpstate{\blocklen}{q_{max}} = 0$} \Comment{No valid prefixes}
     \State   \Return None, $q_{max}$
    \EndIf
    \State $\opt{\vect{r}} \gets \{\}, q_{curr} \gets q_{max}$
    \For{$i \in \{\blocklen,\dots, 1\}$} \Comment{Decoding the optimal string $\opt{\vect{r}}$}
        \State $q_{curr}, t \gets \dppar{i}{q_{curr}}$
        \State $\opt{\vect{r}} \gets \opt{\vect{r}}\concat t$   
    \EndFor
    \State \Return $\text{reverse}(\opt{\vect{r}})$, $q_{max}$         
                
    \end{algorithmic}
\end{algorithm}

\subsection{\Tool algorithm}

Algorithm~\ref{alg:dp_block} presents \Tool{} steps. 
The two main loops dominating its computational complexity involve calculating transition costs and performing the DP updates respectively.

First, for each of the $\blocklen$ time steps, the algorithm computes the optimal single-token transition costs $\costS{i}(q_s, q_t)$ between all source states $q_s \in \dfastates$ and target states $q_t \in \dfastates$. 
This is achieved by iterating through each source state $q_s$, each token $t \in \inalphaLLM{}$, and then for each state $q_t$ reached from $q_s$ via $t$ (i.e., $q_t \in \tokentrans{}(q_s, t)$), updating the cost $\costS{i}(q_s, q_t)$ with $\vect{v}_i[t]$ if it is better. 
The complexity for this part is $O(\blocklen \cdot (|\dfastates|^2 + \sum_{q_s \in \dfastates} \sum_{t \in \inalphaLLM{}} |\tokentrans{}(q_s, t)|))$. 
The sum $\sum_{q_s} \sum_t |\tokentrans{}(q_s,t)|$ represents the total number of transitions, $N_{\text{trans}} = O(|\dfastates| \cdot |\inalphaLLM{}| + |\dfastates| \cdot N_{\bot})$, where $N_{\bot}$ is the maximum number of states reachable via the $\bot$ token. 
Thus, this part takes $O(\blocklen \cdot (|\dfastates|^2 + |\dfastates| \cdot |\inalphaLLM{}|))$.

Second, the core dynamic programming update calculates $\dpstate{i}{q}$ for each diffusion step $i$ and state $q$. This involves iterating over $\blocklen$ diffusion steps, $|\dfastates|$ current states $q$, and for each $q$, considering all $|\dfastates|$ possible previous states $q'$. 
This leads to a complexity of $O(\blocklen \cdot |\dfastates|^2)$.

Combining these dominant parts, the total complexity is $O(\blocklen \cdot (|\dfastates|^2 + |\dfastates| \cdot |\inalphaLLM{}|) + \blocklen \cdot |\dfastates|^2)$, which simplifies to $O(\blocklen \cdot (|\dfastates|^2 + |\dfastates| \cdot |\inalphaLLM{}|))$. 
This can be expressed as $O(\blocklen \cdot |\dfastates| \cdot (|\dfastates| + |\inalphaLLM{}|))$.

\section{Experiments}
\label{sec:experiments}
In this section, we evaluate \Tool{} on a math reasoning task (GSM-Symbolic~\cite{mirzadeh2024gsmsymbolicunderstandinglimitationsmathematical}) and a schema-based text-to-JSON task (JSONModeEval~\citep{jsoneval}) and demonstrate significant improvement over baselines. In both tasks, we use the LLaDA-8B-Base (LLaDA-8B-B)~\cite{llada}, LLaDA-8B-Instruct (LLaDA-8B-I)~\cite{llada}, Dream-v0-Base-7B (Dream-B-7B)~\cite{dream2025}, and Dream-v0-Instruct-7B (Dream-I-7B)~\cite{dream2025} models. 

\noindent \textbf{Experimental Setup.}
We run experiments on a 48-core Intel Xeon Silver 4214R CPU with 2 Nvidia RTX A5000 GPUs. \Tool{} is implemented using PyTorch~\cite{NEURIPS2019_9015} and the HuggingFace transformers library~\cite{wolf-etal-2020-transformers}. 
The token-level DFA is implemented in Rust using a highly efficient regex-DFA library to minimize overhead during DFA construction and LLM inference. 
We report the mean number of DFA states and transitions as well as the offline pre-computation time in Appendix ~\ref{sec:dfa_stats}.

\noindent \textbf{Baselines.}
We compare \Tool{} against unconstrained diffusion LLM generation. Furthermore, to highlight the benefit of optimal constrained decoding with \Tool{}, we implement a constrained decoding strategy \greedyconstrained{} that mirrors existing autoregressive constrained generation methods ~\cite{willard2023efficient, ugare2024syncodellmgenerationgrammar}. \greedyconstrained{} iterates over the diffusion block and at each position $i$ computes a binary mask $m \in \{0, 1\}^{|\inalphaLLM{}|}$ based on the DFA, specifying valid tokens ($m = 1$) and excluded tokens ($m = 0$). Decoding is then performed on the masked probability distribution $m \odot \vect{v_{i}}$, where $\odot$ denotes element-wise multiplication. Since in some cases, \unconstrained{} outperforms \greedyconstrained{}, we also report \bestofbaseline{}, which takes the better result of the two approaches for each problem in the dataset.

\textbf{Math Reasoning: }We evaluate \Tool{} on GSM-Symbolic~\cite{mirzadeh2024gsmsymbolicunderstandinglimitationsmathematical} dataset, which consists of reasoning-based math world problems where numerical values and names are replaced by symbolic variables. Diffusion LLMs are tasked with generating correct symbolic expression solutions to those word problems. We evaluate correctness by using the Z3 solver~\citep{z3} to check if the final expressions from the LLM generations are functionally equivalent to the ground truth expressions. We set the generation length to 128, number of blocks to 8, and total diffusion steps to 64 and prompt the LLMs with 4-shot examples from GSM-Symbolic~\citep{mirzadeh2024gsmsymbolicunderstandinglimitationsmathematical} (the prompts can be found in Appendix ~\ref{sec:gsm_prompts}). We initialize \Tool{} and \greedyconstrained{} with a regex (shown in Appendix ~\ref{sec:gsm_regex}) that permits math expressions wrapped in \textcolor{red}{\texttt{<<}} and  \textcolor{red}{\texttt{>>}} and natural language text outside these expressions for reasoning as done in CRANE \cite{banerjee2025crane}. 

Table~\ref{tab:gsm_symbolic_comparison} compares the performance of \Tool{} with the baseline methods. The Accuracy (\%) column reports the percentage of functionally correct LLM-generated expressions, Parse (\%) indicates the percentage of syntactically valid responses (i.e., expressions without invalid operations), and Time provides the average time in seconds taken to generate a completion. 

As displayed in the table, \Tool{} significantly improves functional correctness over the baselines. For instance, for LLaDA-8B-I, \Tool{} outperforms unconstrained generation by 13 percentage points and \greedyconstrained{} generation by 5 percentage points. Furthermore, $\Tool{}$ achieves 100\% syntactic accuracy across all models evaluated. On the other hand, unconstrained and \greedyconstrained{} generation make many syntactic errors, especially for non-instruct tuned models. For these cases, generation with \greedyconstrained{} results in responses that are syntactically valid prefixes but not syntactically valid by themselves. We present case studies in Appendix ~\ref{sec:gsm_case_studies}. Importantly, \Tool{} is extremely efficient, introducing marginal overhead compared to unconstrained generation.

\begin{table*}[h]
    \centering
    \caption{Comparison of constrained and unconstrained generation methods on GSM-Symbolic.}
    \begin{tabular}{@{}llrrr@{}}
        \toprule
        \textbf{Model} & \textbf{Method} & \textbf{Acc. (\%)} & \textbf{Parse (\%)} &  \textbf{Time (s)} \\
        
\midrule
     & \unconstrained{} & 25 & 54 & 9.06\\
     & \greedyconstrained{} & 30 & 75 & 9.31\\
 LLaDA-8B-B & \bestofbaseline{} & 30 & 75 & 9.08\\
 & \Tool{} & \textbf{31} & \textbf{100} & 9.22\\

\midrule

    & \unconstrained{} & 19 & 35 &  23.78\\
 & \greedyconstrained{} & 27 & 98 & 23.97 \\
 LLaDA-8B-I & \bestofbaseline{} & 27 & 98 & 23.8 \\
 & \Tool{} & \textbf{32} & \textbf{100} & 23.92\\

\midrule

     & \unconstrained{} & 17 & 33 &  16.02\\
   & \greedyconstrained{} & 21 & 41 & 16.13\\
  Dream-B-7B & \bestofbaseline{} & 21 & 41 & 16.04\\
 & \Tool{} & \textbf{23} & \textbf{100} &  16.19\\

 \midrule

     & \unconstrained{} & 32 & 61 & 23.89\\
  & \greedyconstrained{} & 34 & 93 & 24.01 \\
 Dream-I-7B & \bestofbaseline{} & 34 & 93 & 23.9 \\
 & \Tool{} & \textbf{36} & \textbf{100} & 23.91 \\

\bottomrule
    \end{tabular}
    \label{tab:gsm_symbolic_comparison}
\end{table*}

\textbf{JSON Generation: } We further evaluate \Tool{} on a text-to-JSON generation task JSON-Mode-Eval, which conists of zero-shot problems specifying a JSON schema and a request to generate a JSON object that contains specified contents. Generating JSON that adheres to a specified schema is extremely important for applications like tool use and function calling ~\cite{ugare2024syncodellmgenerationgrammar, willard2023efficient}. We evaluate the correctness of JSON generated by an LLM by first evaluating whether the JSON string can be parsed and converted to a valid JSON object. We further evaluate whether the generated JSON is valid against the schema specified in the prompt. We set the generation length to 128, number of blocks to 1, and the total diffusion steps to 64.  For the constrained generation methods, we convert each problem's JSON schema into its corresponding regular expression and guide the diffusion LLM to generate output conforming to that regex.

Table~\ref{tab:json_schema_comparison} presents the results of our experiment. The Parse (\%) column reports the percentage of syntactically valid LLM generations while the Accuracy (\%) column reports the percentage of generations that are both syntactically valid and valid against their respective schemas. Notably, \Tool{} achieves 100\% schema validation and syntactic accuracy, while baseline methods struggle in many cases to generate valid JSON. We attribute this to the fact that \greedyconstrained{} may distort the distribution through its greedy approximation and can only generate a valid prefix, not a fulll parsable generation ~\cite{park2024grammar}. 



\begin{table*}[h]
    \centering
    \caption{Comparison of constrained and unconstrained generation methods for JSON Schema.}
    \begin{tabular}{@{}llrrr@{}}
        \toprule
        \textbf{Model} & \textbf{Method} & \textbf{Acc. (\%)} & \textbf{Parse (\%)} &  \textbf{Time (s)} \\
        
\midrule
    & \unconstrained{} & 57 & 59 &  6.37\\
 & \greedyconstrained{} & 80 & 80 & 6.47 \\
 LLaDA-8B-B & \bestofbaseline{} & 88 & 90 & 6.41 \\
 & \Tool{} & \textbf{100} & \textbf{100} & 6.43 \\

\midrule

    & \unconstrained{} & 87 & 91 & 6.7 \\
 & \greedyconstrained{} & 78 & 79 & 6.81 \\
 LLaDA-8B-I & \bestofbaseline{} & 99 & 99 & 6.73 \\
 & \Tool{} & \textbf{100} & \textbf{100} & 6.78 \\

\midrule

    & \unconstrained{} & 15 & 18 & 5.31 \\
 & \greedyconstrained{} & 23 & 23 & 5.41 \\
 Dream-B-7B & \bestofbaseline{} & 32 & 35 & 5.34 \\
 & \Tool{} & \textbf{100} & \textbf{100} & 5.45 \\

 \midrule

    & \unconstrained{} & 85 & 87 & 6.4 \\
 & \greedyconstrained{} & 30 & 30 & 6.51 \\
 Dream-I-7B & \bestofbaseline{} & 91 & 93 & 6.43 \\
 & \Tool{} & \textbf{100} & \textbf{100} & 6.55 \\

\bottomrule
    \end{tabular}
    \label{tab:json_schema_comparison}
\end{table*}

\textbf{Ablation Study on The Number of Diffusion Blocks: }
We analyze the performance of \Tool{} on GSM-Symbolic using different numbers of diffusion blocks. We run generation with a response length of 128, using 64 total diffusion steps, and each of 1, 2, and 8 blocks. As shown in Figure~\ref{fig:block_gsm}, \Tool{} performs well across all block settings, outperforming baselines in both functional and syntactic correctness. Further ablations on the number of diffusion blocks are presented in Appendix ~\ref{sec:block_abl}. 

\begin{figure}[!b]
\centering\small
 \vspace{-.1in}
\begin{subfigure}[b]{0.45\textwidth}
 \includegraphics[width=\textwidth]{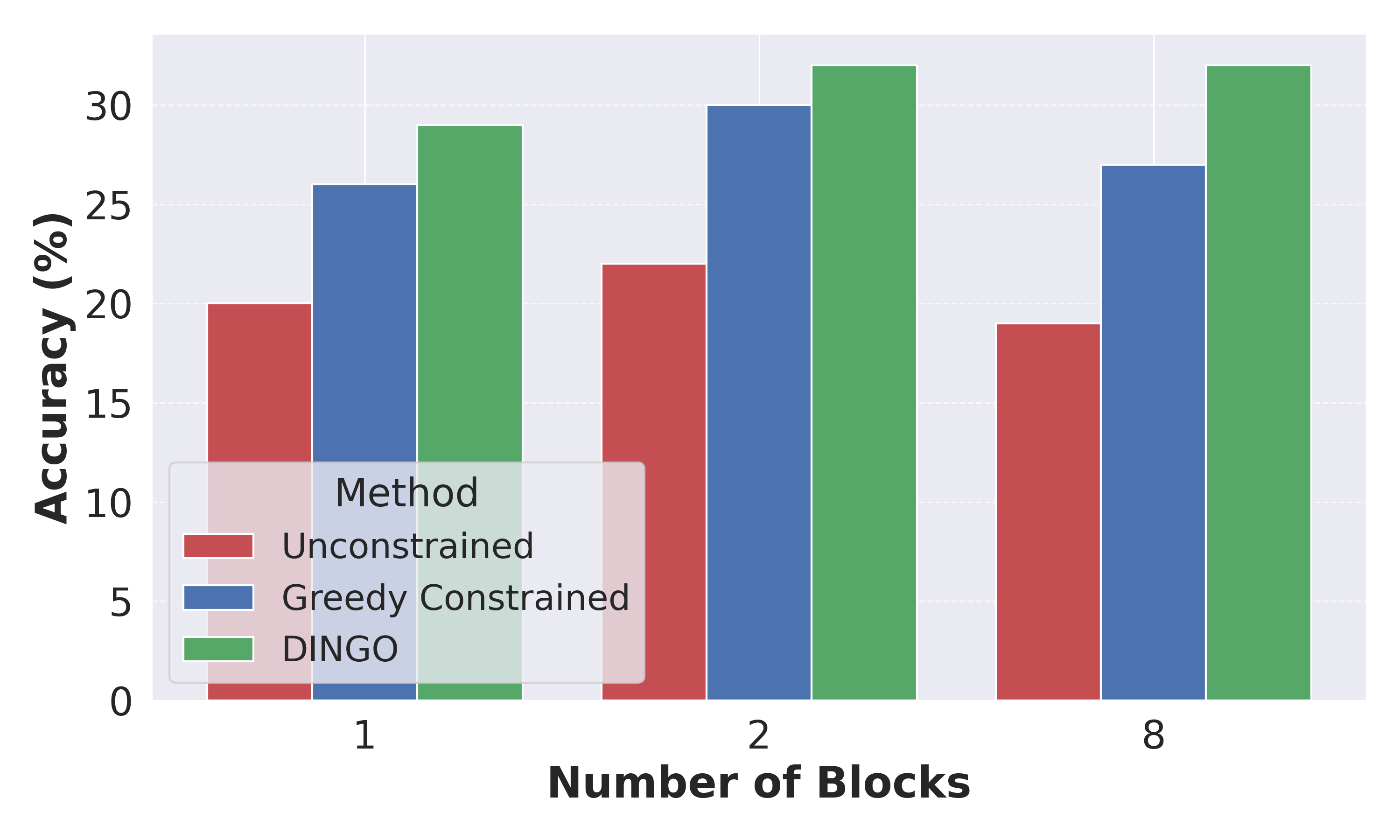}
 \caption{LLaDA-8B-I}
 \label{fig:block_gsm_a}
\end{subfigure}
\hspace{2mm}
\begin{subfigure}[b]{0.45\textwidth}
 \includegraphics[width=\textwidth]{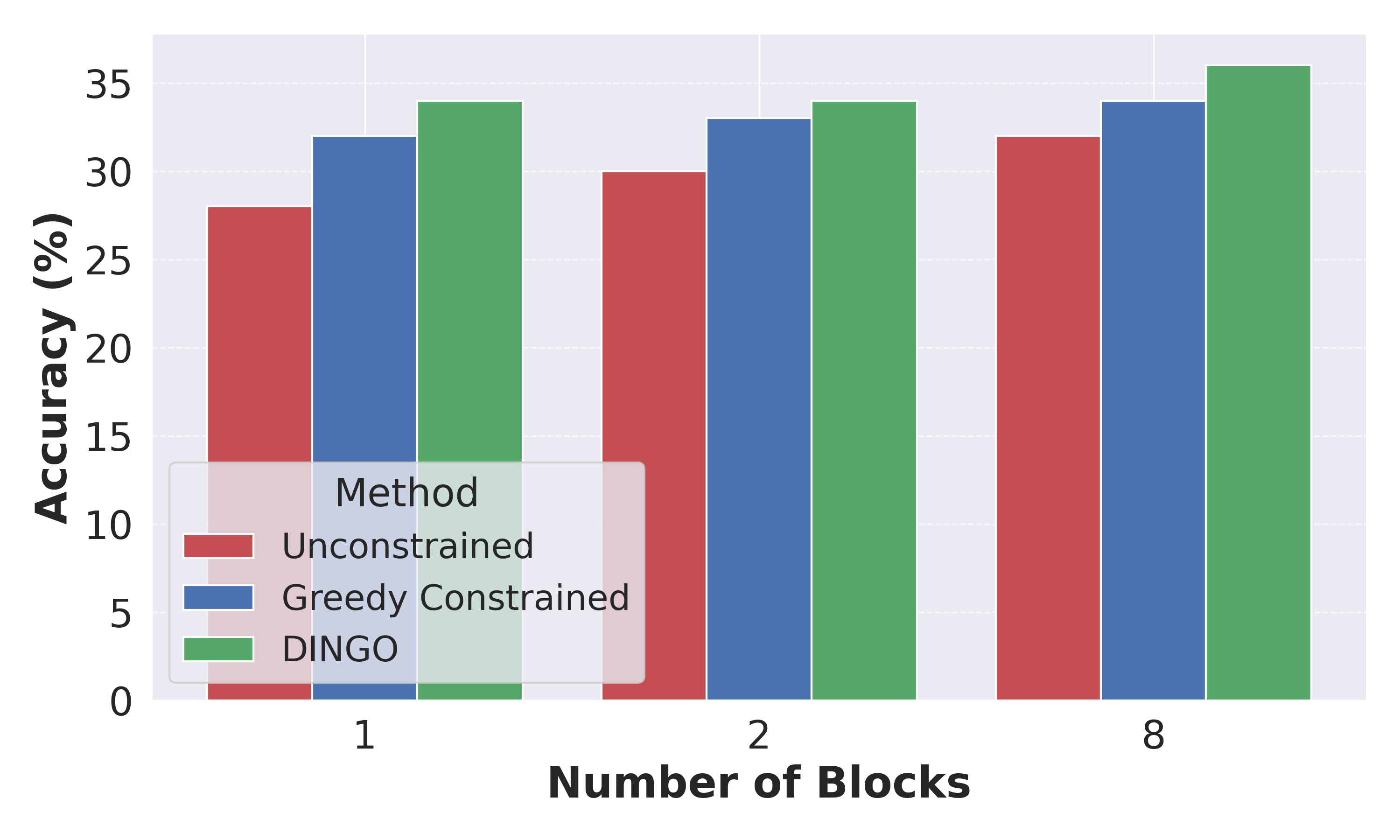}
 \caption{Dream-I-7B}
 \label{fig:block_gsm_b}
\end{subfigure}
\vspace{-.02in}
\caption{Ablation Study on The Number of Diffusion Blocks For GSM-Symbolic}
\label{fig:block_gsm}
\vspace{-.1in}
\end{figure} 





\section{Related Works}
To the best of our knowledge, our work is the first to provide provable guarantees on constrained adherence for inference in diffusion language models.
We next discuss the broader set of related works on diffusion language models and constrained language model decoding.

\noindent\textbf{Diffusion Language Models:}
Diffusion Language Models~\cite{austin2021structured} have emerged as a promising alternative to traditional autoregressive architectures~\cite{radford2019language}, offering advantages in parallel processing and controllability while addressing limitations in sequential generation. Recent advances in semi-autoregressive diffusion models~\cite{han2023ssdlmsemiautoregressivesimplexbaseddiffusion, llada, dream2025, arriola2025block} have significantly narrowed the performance gap with autoregressive counterparts. SSD-LM~\citep{han2023ssdlmsemiautoregressivesimplexbaseddiffusion} introduced a semi-autoregressive approach that performs diffusion over the natural vocabulary space, enabling flexible output length and improved controllability by iteratively generating blocks of text while facilitating local bidirectional context updates.
More recently, several breakthrough models have advanced the field: LLaDA (Large Language Diffusion with mAsking) achieved competitive performance with SOTA open-source autoregressive models of a similar size like LLaMA3-8B through a forward data masking process and a reverse process, parameterized by a vanilla Transformer to predict masked tokens~\citep{llada}. BD3-LMs (Block Discrete Denoising Diffusion Language Models)\cite{arriola2025block} introduced a novel approach that interpolates between discrete denoising diffusion and autoregressive models while supporting flexible-length generation and improving inference efficiency with KV caching. Most recently, Dream-7B\cite{dream2025} emerged as a strong open diffusion large language model that matches state-of-the-art autoregressive (AR) language models of similar size.

\noindent\textbf{Constrained Decoding with Autoregressive LLMs:}
Constrained decoding has shown promising results in augmenting autoregressive language models. Researchers have developed efficient techniques for ensuring syntactic correctness in regular~\citep{deutsch2019general,willard2023efficient,kuchnik2023validating} or context-free \citep{koo2024automata,ugare2024improving,dong2024xgrammar, banerjee2025crane} languages. Other works have focused on semantically constrained decoding through Monte Carlo sampling~\citep{lew2023sequential, loula2025syntactic} or backtracking~\citep{poesia2022synchromesh, ugare2025itergen}. 
\cite{lew2023sequential, park2024grammar} demonstrated that all these approaches that perform greedy constrained approximation for inference can distort the sampling distribution. \Tool{} addresses this challenge by performing optimal constrained sampling on blocks of tokens in a diffusion language model, which partially mitigates distribution distortion issues. 


Concurrent to our work, \cite{cardei2025constrainedlanguagegenerationdiscrete} performs constrained sampling from diffusion language models by minimizing a loss function defined using a surrogate model used for scoring constraints. 
However, their proposed method does not guarantee convergence to the constraint and necessitates a differentiable surrogate model. 
In contrast, our work focuses on providing provable guarantees for constraint satisfaction during inference without the need of an additional surrogate model.



\noindent\textbf{Limitations} 
\Tool is optimal for per-block generation, making it ideal for pure diffusion settings. However, this optimality may not hold in semi-autoregressive setups involving multiple blocks. Currently, our approach is limited to regular language constraints, while programming languages often belong to context-free or context-sensitive classes. As a result, our method cannot directly enforce these more expressive constraints, which have been addressed in prior work on autoregressive constrained generation. Nonetheless, we believe the core dynamic programming framework behind \Tool can be extended to support richer language classes in future work. Moreover, important constraints like toxicity mitigation fall outside formal language classes, highlighting directions for further research.

\vspace{-5pt}
\section{Conclusion}
We presented \Tool, a novel dynamic programming approach that enables diffusion LLMs to generate outputs that strictly adhere to regular language constraints while preserving the model's underlying distribution. 
Our method overcomes the limitations of traditional constrained decoding algorithms that fail with parallel token prediction. 
Our experimental results on symbolic math and JSON generation tasks demonstrate significant improvements over unconstrained inference, demonstrates that \Tool is an effective solution for structured output generation with diffusion models. 
Our work bridges an important gap in making diffusion LLMs reliable for applications requiring formal guarantees.


\clearpage
\newpage
\bibliography{main}

\begin{thebibliography}{31}
\providecommand{\natexlab}[1]{#1}
\providecommand{\url}[1]{\texttt{#1}}
\expandafter\ifx\csname urlstyle\endcsname\relax
  \providecommand{\doi}[1]{doi: #1}\else
  \providecommand{\doi}{doi: \begingroup \urlstyle{rm}\Url}\fi

\bibitem[Arriola et~al.(2025)Arriola, Sahoo, Gokaslan, Yang, Qi, Han, Chiu, and Kuleshov]{arriola2025block}
Marianne Arriola, Subham~Sekhar Sahoo, Aaron Gokaslan, Zhihan Yang, Zhixuan Qi, Jiaqi Han, Justin~T Chiu, and Volodymyr Kuleshov.
\newblock Block diffusion: Interpolating between autoregressive and diffusion language models.
\newblock In \emph{The Thirteenth International Conference on Learning Representations}, 2025.
\newblock URL \url{https://openreview.net/forum?id=tyEyYT267x}.

\bibitem[Austin et~al.(2021)Austin, Johnson, Ho, Tarlow, and van~den Berg]{austin2021structured}
Jacob Austin, Daniel~D. Johnson, Jonathan Ho, Daniel Tarlow, and Rianne van~den Berg.
\newblock Structured denoising diffusion models in discrete state-spaces.
\newblock In A.~Beygelzimer, Y.~Dauphin, P.~Liang, and J.~Wortman Vaughan, editors, \emph{Advances in Neural Information Processing Systems}, 2021.
\newblock URL \url{https://openreview.net/forum?id=h7-XixPCAL}.

\bibitem[Banerjee et~al.(2025)Banerjee, Suresh, Ugare, Misailovic, and Singh]{banerjee2025crane}
Debangshu Banerjee, Tarun Suresh, Shubham Ugare, Sasa Misailovic, and Gagandeep Singh.
\newblock {CRANE}: Reasoning with constrained {LLM} generation.
\newblock \emph{arXiv preprint arXiv:2502.09061}, 2025.
\newblock URL \url{https://arxiv.org/pdf/2502.09061}.

\bibitem[Cardei et~al.(2025)Cardei, Christopher, Hartvigsen, Bartoldson, Kailkhura, and Fioretto]{cardei2025constrainedlanguagegenerationdiscrete}
Michael Cardei, Jacob~K Christopher, Thomas Hartvigsen, Brian~R. Bartoldson, Bhavya Kailkhura, and Ferdinando Fioretto.
\newblock Constrained language generation with discrete diffusion models, 2025.
\newblock URL \url{https://arxiv.org/abs/2503.09790}.

\bibitem[De~Moura and Bj\o{}rner(2008)]{z3}
Leonardo De~Moura and Nikolaj Bj\o{}rner.
\newblock Z3: an efficient smt solver.
\newblock In \emph{Proceedings of the Theory and Practice of Software, 14th International Conference on Tools and Algorithms for the Construction and Analysis of Systems}, TACAS'08/ETAPS'08, page 337–340, Berlin, Heidelberg, 2008. Springer-Verlag.
\newblock ISBN 3540787992.

\bibitem[Deutsch et~al.(2019)Deutsch, Upadhyay, and Roth]{deutsch2019general}
Daniel Deutsch, Shyam Upadhyay, and Dan Roth.
\newblock A general-purpose algorithm for constrained sequential inference.
\newblock In \emph{Proceedings of the Conference on Computational Natural Language Learning}, 2019.
\newblock URL \url{https://aclanthology.org/K19-1045/}.

\bibitem[Dong et~al.(2024)Dong, Ruan, Cai, Lai, Xu, Zhao, and Chen]{dong2024xgrammar}
Yixin Dong, Charlie~F Ruan, Yaxing Cai, Ruihang Lai, Ziyi Xu, Yilong Zhao, and Tianqi Chen.
\newblock {XGrammar}: Flexible and efficient structured generation engine for large language models.
\newblock \emph{arXiv preprint arXiv:2411.15100}, 2024.
\newblock URL \url{https://arxiv.org/pdf/2411.15100}.

\bibitem[et. al.(2021)]{chen2021evaluatinglargelanguagemodels}
Chen et. al.
\newblock Evaluating large language models trained on code, 2021.
\newblock URL \url{https://arxiv.org/abs/2107.03374}.

\bibitem[Fedoseev et~al.(2024)Fedoseev, Dimitrov, Gehr, and Vechev]{fedoseev2024llm}
Timofey Fedoseev, Dimitar~Iliev Dimitrov, Timon Gehr, and Martin Vechev.
\newblock {LLM} training data synthesis for more effective problem solving using satisfiability modulo theories.
\newblock In \emph{The 4th Workshop on Mathematical Reasoning and AI at NeurIPS'24}, 2024.
\newblock URL \url{https://openreview.net/forum?id=hR4Hskr4GX}.

\bibitem[Forney(1973)]{Forney:1973ly}
G.~D. Forney.
\newblock The viterbi algorithm.
\newblock \emph{Proc. of the IEEE}, 61:\penalty0 268 -- 278, March 1973.

\bibitem[Han et~al.(2023)Han, Kumar, and Tsvetkov]{han2023ssdlmsemiautoregressivesimplexbaseddiffusion}
Xiaochuang Han, Sachin Kumar, and Yulia Tsvetkov.
\newblock Ssd-lm: Semi-autoregressive simplex-based diffusion language model for text generation and modular control, 2023.
\newblock URL \url{https://arxiv.org/abs/2210.17432}.

\bibitem[Koo et~al.(2024)Koo, Liu, and He]{koo2024automata}
Terry Koo, Frederick Liu, and Luheng He.
\newblock Automata-based constraints for language model decoding.
\newblock In \emph{Conference on Language Modeling}, 2024.
\newblock URL \url{https://openreview.net/forum?id=BDBdblmyzY}.

\bibitem[Kuchnik et~al.(2023)Kuchnik, Smith, and Amvrosiadis]{kuchnik2023validating}
Michael Kuchnik, Virginia Smith, and George Amvrosiadis.
\newblock Validating large language models with {RELM}.
\newblock \emph{Proceedings of Machine Learning and Systems}, 5, 2023.
\newblock URL \url{https://proceedings.mlsys.org/paper_files/paper/2023/file/93c7d9da61ccb2a60ac047e92787c3ef-Paper-mlsys2023.pdf}.

\bibitem[Lew et~al.(2023)Lew, Zhi-Xuan, Grand, and Mansinghka]{lew2023sequential}
Alexander~K Lew, Tan Zhi-Xuan, Gabriel Grand, and Vikash Mansinghka.
\newblock Sequential {Monte Carlo} steering of large language models using probabilistic programs.
\newblock In \emph{ICML 2023 Workshop: Sampling and Optimization in Discrete Space}, 2023.
\newblock URL \url{https://openreview.net/pdf?id=Ul2K0qXxXy}.

\bibitem[Loula et~al.(2025)Loula, LeBrun, Du, Lipkin, Pasti, Grand, Liu, Emara, Freedman, Eisner, Cotterell, Mansinghka, Lew, Vieira, and O'Donnell]{loula2025syntactic}
Jo{\~a}o Loula, Benjamin LeBrun, Li~Du, Ben Lipkin, Clemente Pasti, Gabriel Grand, Tianyu Liu, Yahya Emara, Marjorie Freedman, Jason Eisner, Ryan Cotterell, Vikash Mansinghka, Alex Lew, Tim Vieira, and Tim O'Donnell.
\newblock Syntactic and semantic control of large language models via sequential {Monte Carlo}.
\newblock In \emph{The Thirteenth International Conference on Learning Representations}, 2025.
\newblock URL \url{https://openreview.net/pdf?id=xoXn62FzD0}.

\bibitem[Mirzadeh et~al.(2024)Mirzadeh, Alizadeh, Shahrokhi, Tuzel, Bengio, and Farajtabar]{mirzadeh2024gsmsymbolicunderstandinglimitationsmathematical}
Iman Mirzadeh, Keivan Alizadeh, Hooman Shahrokhi, Oncel Tuzel, Samy Bengio, and Mehrdad Farajtabar.
\newblock Gsm-symbolic: Understanding the limitations of mathematical reasoning in large language models, 2024.
\newblock URL \url{https://arxiv.org/abs/2410.05229}.

\bibitem[Nie et~al.(2025)Nie, Zhu, You, Zhang, Ou, Hu, Zhou, Lin, Wen, and Li]{llada}
Shen Nie, Fengqi Zhu, Zebin You, Xiaolu Zhang, Jingyang Ou, Jun Hu, Jun Zhou, Yankai Lin, Ji-Rong Wen, and Chongxuan Li.
\newblock Large language diffusion models, 2025.
\newblock URL \url{https://arxiv.org/abs/2502.09992}.

\bibitem[NousResearch(2024)]{jsoneval}
NousResearch.
\newblock json-mode-eval, 2024.
\newblock URL \url{https://huggingface.co/datasets/NousResearch/json-mode-eval}.

\bibitem[Pan et~al.(2023)Pan, Albalak, Wang, and Wang]{pan2023logiclmempoweringlargelanguage}
Liangming Pan, Alon Albalak, Xinyi Wang, and William~Yang Wang.
\newblock Logic-lm: Empowering large language models with symbolic solvers for faithful logical reasoning, 2023.
\newblock URL \url{https://arxiv.org/abs/2305.12295}.

\bibitem[Park et~al.(2024{\natexlab{a}})Park, Wang, Berg-Kirkpatrick, Polikarpova, and D'Antoni]{park2024grammar}
Kanghee Park, Jiayu Wang, Taylor Berg-Kirkpatrick, Nadia Polikarpova, and Loris D'Antoni.
\newblock Grammar-aligned decoding.
\newblock \emph{Advances in Neural Information Processing Systems}, 37:\penalty0 24547--24568, 2024{\natexlab{a}}.
\newblock URL \url{https://proceedings.neurips.cc/paper_files/paper/2024/file/2bdc2267c3d7d01523e2e17ac0a754f3-Paper-Conference.pdf}.

\bibitem[Park et~al.(2024{\natexlab{b}})Park, Wang, Berg-Kirkpatrick, Polikarpova, and D'Antoni]{park2024grammaraligned}
Kanghee Park, Jiayu Wang, Taylor Berg-Kirkpatrick, Nadia Polikarpova, and Loris D'Antoni.
\newblock Grammar-aligned decoding.
\newblock In \emph{The Thirty-eighth Annual Conference on Neural Information Processing Systems}, 2024{\natexlab{b}}.
\newblock URL \url{https://openreview.net/forum?id=5G7ve8E1Lu}.

\bibitem[Paszke et~al.(2019)Paszke, Gross, Massa, Lerer, Bradbury, Chanan, Killeen, Lin, Gimelshein, Antiga, Desmaison, Kopf, Yang, DeVito, Raison, Tejani, Chilamkurthy, Steiner, Fang, Bai, and Chintala]{NEURIPS2019_9015}
Adam Paszke, Sam Gross, Francisco Massa, Adam Lerer, James Bradbury, Gregory Chanan, Trevor Killeen, Zeming Lin, Natalia Gimelshein, Luca Antiga, Alban Desmaison, Andreas Kopf, Edward Yang, Zachary DeVito, Martin Raison, Alykhan Tejani, Sasank Chilamkurthy, Benoit Steiner, Lu~Fang, Junjie Bai, and Soumith Chintala.
\newblock Pytorch: An imperative style, high-performance deep learning library.
\newblock In \emph{Advances in Neural Information Processing Systems 32}. 2019.

\bibitem[Poesia et~al.(2022)Poesia, Polozov, Le, Tiwari, Soares, Meek, and Gulwani]{poesia2022synchromesh}
Gabriel Poesia, Alex Polozov, Vu~Le, Ashish Tiwari, Gustavo Soares, Christopher Meek, and Sumit Gulwani.
\newblock Synchromesh: Reliable code generation from pre-trained language models.
\newblock In \emph{International Conference on Learning Representations}, 2022.
\newblock URL \url{https://openreview.net/forum?id=KmtVD97J43e}.

\bibitem[Radford et~al.(2019)Radford, Wu, Child, Luan, Amodei, and Sutskever]{radford2019language}
Alec Radford, Jeffrey Wu, Rewon Child, David Luan, Dario Amodei, and Ilya Sutskever.
\newblock Language models are unsupervised multitask learners.
\newblock \emph{OpenAI}, 2019.
\newblock URL \url{https://cdn.openai.com/better-language-models/language_models_are_unsupervised_multitask_learners.pdf}.
\newblock Accessed: 2024-11-15.

\bibitem[Ugare et~al.(2024{\natexlab{a}})Ugare, Suresh, Kang, Misailovic, and Singh]{ugare2024improving}
Shubham Ugare, Tarun Suresh, Hangoo Kang, Sasa Misailovic, and Gagandeep Singh.
\newblock {SynCode}: Improving {LLM} code generation with grammar augmentation.
\newblock \emph{arXiv preprint arXiv:2403.01632}, 2024{\natexlab{a}}.
\newblock URL \url{https://arxiv.org/pdf/2403.01632}.

\bibitem[Ugare et~al.(2024{\natexlab{b}})Ugare, Suresh, Kang, Misailovic, and Singh]{ugare2024syncodellmgenerationgrammar}
Shubham Ugare, Tarun Suresh, Hangoo Kang, Sasa Misailovic, and Gagandeep Singh.
\newblock Syncode: Llm generation with grammar augmentation, 2024{\natexlab{b}}.
\newblock URL \url{https://arxiv.org/abs/2403.01632}.

\bibitem[Ugare et~al.(2025)Ugare, Gumaste, Suresh, Singh, and Misailovic]{ugare2025itergen}
Shubham Ugare, Rohan Gumaste, Tarun Suresh, Gagandeep Singh, and Sasa Misailovic.
\newblock {IterGen}: Iterative structured {LLM} generation.
\newblock In \emph{The Thirteenth International Conference on Learning Representations}, 2025.
\newblock URL \url{https://openreview.net/pdf?id=ac93gRzxxV}.

\bibitem[Willard and Louf(2023)]{willard2023efficient}
Brandon~T Willard and R{\'e}mi Louf.
\newblock Efficient guided generation for large language models.
\newblock \emph{arXiv preprint arXiv:2307.09702}, 2023.
\newblock URL \url{https://arxiv.org/pdf/2307.09702}.

\bibitem[Wolf et~al.(2020)Wolf, Debut, Sanh, Chaumond, Delangue, Moi, Cistac, Rault, Louf, Funtowicz, Davison, Shleifer, von Platen, Ma, Jernite, Plu, Xu, Le~Scao, Gugger, Drame, Lhoest, and Rush]{wolf-etal-2020-transformers}
Thomas Wolf, Lysandre Debut, Victor Sanh, Julien Chaumond, Clement Delangue, Anthony Moi, Pierric Cistac, Tim Rault, Remi Louf, Morgan Funtowicz, Joe Davison, Sam Shleifer, Patrick von Platen, Clara Ma, Yacine Jernite, Julien Plu, Canwen Xu, Teven Le~Scao, Sylvain Gugger, Mariama Drame, Quentin Lhoest, and Alexander Rush.
\newblock Transformers: State-of-the-art natural language processing.
\newblock In \emph{Conference on Empirical Methods in Natural Language Processing: System Demonstrations}, 2020.
\newblock URL \url{https://aclanthology.org/2020.emnlp-demos.6}.

\bibitem[Yang et~al.(2023)Yang, Swope, Gu, Chalamala, Song, Yu, Godil, Prenger, and Anandkumar]{yang2023leandojotheoremprovingretrievalaugmented}
Kaiyu Yang, Aidan~M. Swope, Alex Gu, Rahul Chalamala, Peiyang Song, Shixing Yu, Saad Godil, Ryan Prenger, and Anima Anandkumar.
\newblock Leandojo: Theorem proving with retrieval-augmented language models, 2023.
\newblock URL \url{https://arxiv.org/abs/2306.15626}.

\bibitem[Ye et~al.(2025)Ye, Xie, Zheng, Gao, Wu, Jiang, Li, and Kong]{dream2025}
Jiacheng Ye, Zhihui Xie, Lin Zheng, Jiahui Gao, Zirui Wu, Xin Jiang, Zhenguo Li, and Lingpeng Kong.
\newblock Dream 7b, 2025.
\newblock URL \url{https://hkunlp.github.io/blog/2025/dream}.

\end{thebibliography}
\bibliographystyle{plainnat}

\newpage
\section*{NeurIPS Paper Checklist}

The checklist is designed to encourage best practices for responsible machine learning research, addressing issues of reproducibility, transparency, research ethics, and societal impact. Do not remove the checklist: {\bf The papers not including the checklist will be desk rejected.} The checklist should follow the references and follow the (optional) supplemental material.  The checklist does NOT count towards the page
limit. 

Please read the checklist guidelines carefully for information on how to answer these questions. For each question in the checklist:
\begin{itemize}
    \item You should answer \answerYes{}, \answerNo{}, or \answerNA{}.
    \item \answerNA{} means either that the question is Not Applicable for that particular paper or the relevant information is Not Available.
    \item Please provide a short (1–2 sentence) justification right after your answer (even for NA). 
\end{itemize}

{\bf The checklist answers are an integral part of your paper submission.} They are visible to the reviewers, area chairs, senior area chairs, and ethics reviewers. You will be asked to also include it (after eventual revisions) with the final version of your paper, and its final version will be published with the paper.

The reviewers of your paper will be asked to use the checklist as one of the factors in their evaluation. While "\answerYes{}" is generally preferable to "\answerNo{}", it is perfectly acceptable to answer "\answerNo{}" provided a proper justification is given (e.g., "error bars are not reported because it would be too computationally expensive" or "we were unable to find the license for the dataset we used"). In general, answering "\answerNo{}" or "\answerNA{}" is not grounds for rejection. While the questions are phrased in a binary way, we acknowledge that the true answer is often more nuanced, so please just use your best judgment and write a justification to elaborate. All supporting evidence can appear either in the main paper or the supplemental material, provided in appendix. If you answer \answerYes{} to a question, in the justification please point to the section(s) where related material for the question can be found.

IMPORTANT, please:
\begin{itemize}
    \item {\bf Delete this instruction block, but keep the section heading ``NeurIPS Paper Checklist"},
    \item  {\bf Keep the checklist subsection headings, questions/answers and guidelines below.}
    \item {\bf Do not modify the questions and only use the provided macros for your answers}.
\end{itemize}


\begin{enumerate}

\item {\bf Claims}
    \item[] Question: Do the main claims made in the abstract and introduction accurately reflect the paper's contributions and scope?
    \item[] Answer: \answerYes{} 
    \item[] Justification: Yes, the abstract and introductions claims are validated through the proposed approach and experimental results. 

\item {\bf Limitations}
    \item[] Question: Does the paper discuss the limitations of the work performed by the authors?
    \item[] Answer: \answerYes{} 
    \item[] Justification: Limitations are discussed in the Related Works section. 

\item {\bf Theory assumptions and proofs}
    \item[] Question: For each theoretical result, does the paper provide the full set of assumptions and a complete (and correct) proof?
    \item[] Answer: \answerYes{} 
    \item[] Justification: Yes, refer to Section ~\ref{sec:algoDetail}.

    \item {\bf Experimental result reproducibility}
    \item[] Question: Does the paper fully disclose all the information needed to reproduce the main experimental results of the paper to the extent that it affects the main claims and/or conclusions of the paper (regardless of whether the code and data are provided or not)?
    \item[] Answer: \answerYes{} 
    \item[] Justification: All experiment details are in Section ~\ref{sec:experiments}. 

\item {\bf Open access to data and code}
    \item[] Question: Does the paper provide open access to the data and code, with sufficient instructions to faithfully reproduce the main experimental results, as described in supplemental material?
    \item[] Answer:  \answerYes{}
    \item[] Justification: Code is provided. 

\item {\bf Experimental setting/details}
    \item[] Question: Does the paper specify all the training and test details (e.g., data splits, hyperparameters, how they were chosen, type of optimizer, etc.) necessary to understand the results?
    \item[] Answer: \answerYes{} 
    \item[] Justification: All experiment details are mentioned in Section ~\ref{sec:experiments}.

\item {\bf Experiment statistical significance}
    \item[] Question: Does the paper report error bars suitably and correctly defined or other appropriate information about the statistical significance of the experiments?
    \item[] Answer: \answerNo{} 
    \item[] Justification: All experiments used greedy decoding, which is deterministic. The results will not change between runs. 

\item {\bf Experiments compute resources}
    \item[] Question: For each experiment, does the paper provide sufficient information on the computer resources (type of compute workers, memory, time of execution) needed to reproduce the experiments?
    \item[] Answer: \answerYes{} 
    \item[] Justification: Yes, the details are provided in Section ~\ref{sec:experiments}. 
    
\item {\bf Code of ethics}
    \item[] Question: Does the research conducted in the paper conform, in every respect, with the NeurIPS Code of Ethics \url{https://neurips.cc/public/EthicsGuidelines}?
    \item[] Answer: \answerYes{} 
    \item[] Justification: The paper conforms to NeurIPS Code of Ethics

\item {\bf Broader impacts}
    \item[] Question: Does the paper discuss both potential positive societal impacts and negative societal impacts of the work performed?
    \item[] Answer: \answerYes{} 
    \item[] Justification: The paper discusses the impact of the work in the Introduction and validate through the experiments.  
    
\item {\bf Safeguards}
    \item[] Question: Does the paper describe safeguards that have been put in place for responsible release of data or models that have a high risk for misuse (e.g., pretrained language models, image generators, or scraped datasets)?
    \item[] Answer: \answerNA{} 
    \item[] Justification: No such risks. 

\item {\bf Licenses for existing assets}
    \item[] Question: Are the creators or original owners of assets (e.g., code, data, models), used in the paper, properly credited and are the license and terms of use explicitly mentioned and properly respected?
    \item[] Answer: \answerYes{} 
    \item[] Justification: All citations for assets used are given. 

\item {\bf New assets}
    \item[] Question: Are new assets introduced in the paper well documented and is the documentation provided alongside the assets?
    \item[] Answer: \answerYes{}
    \item[] Justification: The details about the algorithm and limitations are described in the paper

\item {\bf Crowdsourcing and research with human subjects}
    \item[] Question: For crowdsourcing experiments and research with human subjects, does the paper include the full text of instructions given to participants and screenshots, if applicable, as well as details about compensation (if any)? 
    \item[] Answer: \answerNA{} 
    \item[] Justification: No crowdsourcing nor research with human subjects.

\item {\bf Institutional review board (IRB) approvals or equivalent for research with human subjects}
    \item[] Question: Does the paper describe potential risks incurred by study participants, whether such risks were disclosed to the subjects, and whether Institutional Review Board (IRB) approvals (or an equivalent approval/review based on the requirements of your country or institution) were obtained?
    \item[] Answer: \answerNA{} 
    \item[] Justification: No human subject involved. 

\item {\bf Declaration of LLM usage}
    \item[] Question: Does the paper describe the usage of LLMs if it is an important, original, or non-standard component of the core methods in this research? Note that if the LLM is used only for writing, editing, or formatting purposes and does not impact the core methodology, scientific rigorousness, or originality of the research, declaration is not required.
    \item[] Answer: \answerNo{} 
    \item[] Justification: LLM is not used for important, original, or non-standard component of the research. 

\end{enumerate}




\appendix
\section{Transformer and Remasking Step}
\label{appen:tansformAndRemask}
\begin{algorithm}[H]
    \caption{Diffusion Step}
    \label{alg:diffusionStep}
    \begin{algorithmic}[1]
    \Require Transformer $\transform{n}$, full output length $n$, prompt $\vect{p}$, input prompt length $m$, block length $\blocklen$, current diffusion step $i$, total diffusion steps $T$, vocabulary $V$.
    \State $\vect{x} \leftarrow \vect{p}\concat\bot^{\blocklen}$ \Comment{Pad the input prompt where $n = m + \blocklen$}
    \State $\vect{v_1}\dots\vect{v_n} \gets \transform{n}(\vect{x})$ \Comment{$\vect{v_i} \in \mathbb{R}^{|\inalphaLLM{}|}_{+}$ output distribution at position $i$}
    \State $\vect{l} \gets \text{RemaskPositions}(\vect{v}_{m+1}, \dots, \vect{v}_{m+d}, i, T)$ \Comment{Decides which positions to remask}
    \For{$j \in \vect{l}$}
    \State $\vect{v_j} \gets \vect{0}$
    \State $\vect{v_j}[\bot] \gets 1$ \Comment{Set probability of all tokens except $\bot$ to $0$.}
    \EndFor
    \State $\vect{r} \gets \decode{n}(\vect{v_1}\dots\vect{v_n})$ \Comment{Decoding that outputs response with first $m$ tokens are input prompt $\vect{p}$} 
    \State \Return \vect{r}
    \end{algorithmic}
\end{algorithm}

We describe the two key components of a single diffusion step:
a) \textbf{Transformer step:} Computes the output distribution over all tokens in the vocabulary (line 2 Algo.~\ref{alg:diffusionStep}).
b) \textbf{Remasking step:} Based on the output from the transformer step, it greedily decides which token positions to mask. The remasking step can be viewed as updating the output distribution such that, at the masked positions, the mask token $\bot$ is assigned probability 1, while all other tokens receive probability 0 (lines 3 -- 6 Algo.~\ref{alg:diffusionStep}).
Popular greedy remasking strategies include (line 3 Algo.~\ref{alg:diffusionStep}):
(i) \textit{Random:} Masks tokens at randomly selected positions \cite{llada}.
(ii) \textit{Top token probability:} Masks positions where the top-predicted token has the lowest probability \cite{llada}.
(iii) \textit{Entropy-based:} Computes the entropy of the output distribution at each position and masks the positions with the highest entropy \cite{dream2025}.

The number of token positions to remask at the $i$-th step typically depends on the total number of diffusion steps $T$ and the block length $\blocklen$. At step $0$, all $\blocklen$ positions are masked, and the number of masked tokens decreases linearly to 0 over $T$ steps. Thus, at the $i$-th step, the number of masked tokens is given by $\left\lfloor \frac{\blocklen \times (T - i)}{T} \right\rfloor$.

\section{Proofs}
\label{appen:proofs}
\correctnessThm*

\begin{proof}
We assume that $\exists \vect{x} \in L_P(\regex) \cap \nomask^{\blocklen} \wedge \left(P(\vect{x} | \vect{v}_{m+1}\dots \vect{v}_{m+d})\geq 0\right)$ then the decoded string $\vect{r}$ satisfy the soundness property (see Definition~\ref{def:correctness}). In other words, if there is at least one fully unmasked valid prefix with non-zero probability then \Tool retrieves a valid string. 

\noindent We show this by induction on the position of tokens. Before moving to the proof, we first define extended transition function $\transitions{}^*$ when $\transitions{}: \dfastates \times \inalphaLLM{} \to \expstate$ outputs a set of states instead of single state due to mask token $\bot$. In this case, for any string $\vect{w}\in \inalphaLLM{*}$, $\transitions{}^{*}(\vect{w}, q_0)$ represents the state of reachable states starting from $q_0$. This can be defined as $\transitions{}^{*}(\{\}, q_0) = \{q_0\}$ and $\transitions{}^{*}(t_1\cdots t_{m+1}, q_0) = \cup_{q \in \transitions{}^{*}(t_1\cdots t_m, q_0)} \transitions{}(q, t_{m+1})$. 

\begin{enumerate}
\item Let $0 \leq i \leq d$, and let $t_1 \dots t_i \in \inalphaLLM{i}$ denote any token sequence with positive probability mass $\prod_{j=1}^{i} \vect{v}_{m+j}[t_j] > 0$. Let $q \in \tokentrans{}^{*}( t_{1}\cdots t_i, \dfastart)$. Then, $\dpstate{i}{q} > 0$. We prove this using induction on $i$. 
\begin{enumerate}
\item Base case $i = 0$: For empty strings only start state $q_0$ is reachable. \Tool initializes $\dpstate{0}{\dfastart} = 1 > 0$ and for all $q \neq q_0$, $\dpstate{0}{q} = 0$. (lines 1 -- 3 in Algo.~\ref{alg:dp_block}). 

\item Inductive Step: 
At position $i + 1$, let $t_1 \dots t_{i+1} \in \inalphaLLM{i+1}$ s.t. $\prod_{j=1}^{i + 1} \vect{v}_{m+j}[t_j] > 0$. Let $q' \in \tokentrans{}^{*}( t_{1}\cdots t_i, \dfastart)$ and $q \in \tokentrans{}(q', t_{i+1})$. By the inductive hypothesis, for all such $q'$ $\dpstate{i}{q'} > 0$. 
Recall, 
\begin{align*}
\cost{q}{q'}{i+1} = \begin{cases}
    \max\limits_{t \in \inalphaLLM{}}\; \vect{v}_{m + i+1}(t) \text{ s.t. $q \in \tokentrans{}(q', t)$} \\
    0  \text{  if $q, q'$ are not connected}   
\end{cases} & \dpstate{i+1}{q} = \max_{q'\in \dfastates} \dpstate{i}{q'} \times \cost{q}{q'}{i+1}
\end{align*}

Thus, $\cost{q}{q'}{i+1} \geq \vect{v}_{m + i+1}(t_{i+1}) > 0$ which implies $\dpstate{i}{q'} \times \cost{q}{q'}{i+1} > 0$. 
Therefore, $\dpstate{i+1}{q} = \max_{q'\in \dfastates} \dpstate{i}{q'} \times \cost{q}{q'}{i+1} > 0$.
\end{enumerate}

\item Since $\pre{\regex} \cap \nomask^{d} \neq \{\}$ by assumption, there exists some $\vect{y} \in \pre{\regex} \cap \nomask^{d}$. By the Definition ~\ref{def:liveState}, $q_l = \tokentrans{t}^{*}(\vect{y}, \dfastart) \in \live{Q}$. From the induction above, $\dpstate{d}{q_l} > 0$. From line 16 in Algo.~\ref{alg:dp_block}, $q_{max} = \argmax_{q\in Q_l} \dpstate{\blocklen}{q}$. Thus, by the definition of $\argmax$, $\dpstate{d}{q_{max}} \geq \dpstate{d}{q_l} > 0$. 

\item In lines 20-22 in Algo.~\ref{alg:dp_block}), \Tool{} reconstructs a d-length sequence $r = t_1 \dots t_d \in \inalphaLLM{d}$ such that $q_{max} \in \tokentrans{}^{*}(r, \dfastart)$. 
For any $t_j \in r$, if $t_j = \bot$, choose any token $\tau_j \in \nomask$ satisfying $\tokentrans{t}( q_{j-1}, \tau_j) = q_j$ where $q_j = \tokentrans{t}^{*}( t_1 \dots t_j, \dfastart)$. By definition of $\tokentrans{\bot}$, $\tau_j$ exists. Substituting every $\bot$ in this manner yields, by Definition ~\ref{def:substitutionset}, $\vect{x} = \vect{x_1} \dots \vect{x_d} \in \nomask^{d}$. $\vect{x} \in \sub{\vect{r}}$. $\tokentrans{t}^{*}(\vect{x} , \dfastart) = q_{max}$. From above, $\dpstate{d}{q_{max}} > 0$.  

\item Since $q_{max} \in Q_l$, by Definition ~\ref{def:liveState}, $\exists w \in \alphabets^*$ s.t.  $\transitions^*(\vect{w}, q_{max}) \in \dfafinal$. Equivalently, $\vect{x} \concat w \in \lang{\regex}$, hence $\vect{x}  \in \pre{\regex}$. 
\end{enumerate}
\end{proof}

\optimalityThm*
\begin{proof}
\begin{enumerate}
\item First, we show that $P(\opt{\vect{r}} \;|\; \vect{v}_{m+1}\dots\vect{v}_{n}) = \dpstate{d}{q_{max}}$, or equivalently $\prod_{j=1}^{d} \vect{v}_{m+j}[\opt{\vect{r_j}}] = \dpstate{d}{q_{max}}$. Let $\opt{\vect{r}} = \opt{\vect{r_1}} \dots \opt{\vect{r_d}}$ and $0 \leq i \leq d$. We prove by induction on $i$ that if \Tool's backtracking (lines 19 -- 23 in Algo ~\ref{alg:dp_block}) has brought us to state $q \in Q$ at position $i$, then $\dpstate{i}{q} = \prod_{j=1}^{i} \vect{v}_{m+j}[\opt{\vect{r_j}}]$.
\begin{enumerate}
\item Base case i = 0: $\dpstate{0}{\dfastart} = 1 = \prod_{j=1}^{0} \vect{v}_{m+j}[\opt{\vect{r_j}}]$.
\item Inductive Step: At position $i$, let $q', \opt{\vect{r_i}} = \dppar{i}{q}$ (line 21 in Algo ~\ref{alg:dp_block}). From lines 14 -- 15 in Algo ~\ref{alg:dp_block}, $\dpstate{i}{q} = \dpstate{i - 1}{q'} \times \vect{v}_{m + i}(\opt{\vect{r_i}}) $. By the inductive hypothesis, $\dpstate{i - 1}{q'} = \prod_{j=1}^{i - 1} \vect{v}_{m+j}[\opt{\vect{r_j}}]$. Thus, $\dpstate{i}{q} = \prod_{j=1}^{i - 1} \vect{v}_{m+j}[\opt{\vect{r_j}}] \times \vect{v}_{m + i}(\opt{\vect{r_i}}) = \prod_{j=1}^{i} \vect{v}_{m+j}[\opt{\vect{r_j}}]$.
\end{enumerate}
Let $q_d \in \tokentrans{}^{*}(\opt{\vect{r_1}} \dots \opt{\vect{r_d}}, \dfastart)$. Since $q_d = q_{max}$ (line 19 in Algo ~\ref{alg:dp_block}), $\dpstate{d}{q_{max}} = \prod_{j=1}^{d} \vect{v}_{m+j}[\opt{\vect{r_j}}] = P(\opt{\vect{r}} \;|\; \vect{v}_{m+1}\dots\vect{v}_{n})$.

\item We show that for every valid string $\vect{r}' = \vect{r_1}' \dots \vect{r_d}'$ satisfying $\exists \vect{x} \in V^*. (\vect{x} \in \sub{\vect{r'}}) \wedge (\vect{x} \in \pre{\regex})$, $\prod_{j=1}^{d} \vect{v}_{m+j}[\vect{r_j}'] \leq \dpstate{d}{q_{max}}$. Let $0 \leq i \leq d$ and $q \in \tokentrans{}^{*}( \vect{r_1}' \dots \vect{r_i}', \dfastart)$. We show that $\prod_{j=1}^{i} \vect{v}_{m+j}[\vect{r_j}'] \leq \dpstate{d}{q}$ using induction on $i$.
\begin{enumerate}
\item Base case i = 0: $\dpstate{0}{\dfastart} = 1 = \prod_{j=1}^{0} \vect{v}_{m+j}[\vect{r_j}']$.
\item Inductive Step: At position $i + 1$, let $q' \in \tokentrans{}^{*}( \vect{r_1}' \cdots \vect{r_i}', \dfastart)$ and $q \in \tokentrans{}(q', \vect{r_{i + 1}}')$. By the inductive hypothesis, $\prod_{j=1}^{i} \vect{v}_{m+j}[\vect{r_j}'] \leq \dpstate{i}{q'}$. 
Recall, 
\begin{align*}
\cost{q}{q'}{i+1} = \begin{cases}
    \max\limits_{t \in \inalphaLLM{}}\; \vect{v}_{m + i+1}(t) \text{ s.t. $q \in \tokentrans{}(q', t)$} \\
    0  \text{  if $q, q'$ are not connected}   
\end{cases} & \dpstate{i+1}{q} = \max_{q'\in \dfastates} \dpstate{i}{q'} \times \cost{q}{q'}{i+1}
\end{align*}
Thus, $\vect{v}_{m + i+1}(\vect{r_{i + 1}}') \leq \cost{q}{q'}{i+1}$. Hence, $\prod_{j=1}^{i + 1} \vect{v}_{m+j}[\vect{r_j}'] = \prod_{j=1}^{i} \vect{v}_{m+j}[\vect{r_j}'] \times \vect{v}_{m + i+1}(\vect{r_{i + 1}}') \leq \dpstate{i}{q'} \times \cost{q}{q'}{i+1} \leq \dpstate{i + 1}{q}$. 

\end{enumerate}
Let $q_d \in \tokentrans{}^{*}( \vect{r_1}' \dots \vect{r_d}', \dfastart)$.
Since $\vect{x} \in V^*. (\vect{x} \in \sub{\vect{r'}}) \wedge (\vect{x} \in \pre{\regex})$, $q_d \in Q_l$. From line 16 in Algo.~\ref{alg:dp_block}, $q_{max} = \argmax_{q\in Q_l} \dpstate{\blocklen}{q}$. Thus, by the definition of $\argmax$, $\dpstate{d}{q_d} \leq \dpstate{d}{q_{max}}$. From the inductive hypothesis above, $\prod_{j=1}^{d} \vect{v}_{m+j}[\vect{r_j}'] \leq \dpstate{d}{q_d} \leq \dpstate{d}{q_{max}}$.

\item Hence, $P(\vect{r'} \;|\; \vect{v}_{m+1}\dots\vect{v}_{n}) = \prod_{j=1}^{d} \vect{v}_{m+j}[\vect{r_j}'] \leq \dpstate{d}{q_{max}} = \prod_{j=1}^{d} \vect{v}_{m+j}[\opt{\vect{r_j}}] = P(\opt{\vect{r}} \;|\; \vect{v}_{m+1}\dots\vect{v}_{n})$. 
\end{enumerate}
\end{proof}

\section{Time complexity analysis of parallelized DINGO DP}
\label{appen:complexity}
\begin{algorithm}[H]
    \small
    \caption{\Tool{} DP}
    \label{alg:dpParallel}
    \begin{algorithmic}[1]
    \Require $\dfastart$, block length $\blocklen$, probability vectors $\vect{v}_1, \dots \vect{v}_{\blocklen}$ for the current block,  $\live{Q}$, $\dfastates$, $\tokentrans{}$.
    \State $\dpstate{0}{q} \gets 0$ for all $(q \in \dfastates) \wedge (q \neq q_0)$
    \State $\dpstate{0}{\dfastart} \gets 1$
    \State $\dppar{0}{q} \gets (\text{None}, \text{None})$  for all $(q \in \dfastates)$ \Comment{Initialization of the DP}
    \State $\costS{i} \gets \{\}$ for all $i \in \{1,\dots,\blocklen\}$\Comment{maximum token probability transtion $(q' \to  q)$ at position $i$}
    \State $\tranS{i} \gets \{\}$ for all $i \in \{1,\dots,\blocklen\}$ \Comment{token for the maximum probability transition $(q' \to q)$}
    \For{\textcolor{blue}{$i \in \{1,\dots, \blocklen\}$}} \Comment{\textcolor{blue}{The computation along all $\blocklen$ can be parallelized}}
    \State \textcolor{blue}{\# Parallelize for each $\{1, \dots \blocklen\}$}
    \For{$(q \in \dfastates)$}
    \For{$t \in \inalphaLLM{}$}
    \State $q' \gets \delta(q, t)$
    \State $\costS{i}(q, q'), \tranS{i}(q, q') \gets $ MaxTransition$(\vect{v}_i, t, q, q')$   
    \EndFor
    \EndFor
    \EndFor
    \For{$i \in \{1,\dots, \blocklen\}$} \Comment{DP computation loop}
    \For{$(q \in \dfastates)\wedge (q' \in \dfastates)$}
        \If{$\dpstate{i}{q} < \dpstate{i-1}{q'}\times \costS{i}(q, q')$ }
        \State $\dpstate{i}{q} \gets \dpstate{i-1}{q'}\times \costS{i}(q, q')$ \Comment{Update maximum probability path to $q$}
        \State $\dppar{i}{q} \gets (q', \tranS{i}(q, q'))$  \Comment{Update the parents accordingly}
        \EndIf
    \EndFor
    \EndFor
    \State $q_{max} \gets \argmax_{q\in Q_l} \dpstate{\blocklen}{q}$
    \If{$\dpstate{\blocklen}{q_{max}} = 0$} \Comment{No valid prefixes}
     \State   \Return None, $q_{max}$
    \EndIf
    \State $\opt{\vect{r}} \gets \{\}, q_{curr} \gets q_{max}$
    \For{$i \in \{\blocklen,\dots, 1\}$} \Comment{Decoding the optimal string $\opt{\vect{r}}$}
        \State $q_{curr}, t \gets \dppar{i}{q_{curr}}$
        \State $\opt{\vect{r}} \gets \opt{\vect{r}}\concat t$   
    \EndFor
    \State \Return $\text{reverse}(\opt{\vect{r}})$, $q_{max}$         
                
    \end{algorithmic}
\end{algorithm}
The parallelism step at line 6 in Algo.~\ref{alg:dpParallel} can be efficiently implemented using popular frameworks like PyTorch. With parallelism, the computational depth (i.e., the minimum number of sequential steps) reduces to $O(\max(|\dfastates|^2, |\dfastates| \times |\inalphaLLM{}|) + |\dfastates|^2 \times \blocklen)$. For regular expressions, where the number of states $|\dfastates|$ is a small constant, the computational depth becomes $O(|\inalphaLLM{}| + \blocklen)$, which is linear in both the vocabulary size $|\inalphaLLM{}|$ and the block length $\blocklen$.   

\section{Semi-Autoregressive}
\label{appen:semiauto}
In the semi-autoregressive setup, given an input $\vect{p} \in \inalphaLLM{m}$, the output $\vect{o} \in \inalphaLLM{m + \blocklen \times k}$ is generated over $k$ blocks, where each block is computed via a call to the single block diffusion model. The output of the $i$-th diffusion model call is $\vect{x}_i = \llm{m_i, n_i}(\vect{x}_{i-1})$, with $\vect{x}_0 = \vect{p}$ and the final output $\vect{o} = \vect{x}_k$. The input and output lengths for each block are defined as $m_i = m + (i - 1) \times \blocklen$ and $n_i = m + i \times \blocklen$ for all $1 \leq i \leq k$.
\begin{algorithm}[H]
    \caption{Semi-Autoregressive diffusion LLM Generation}
    \label{alg:dllm_gen}
    \begin{algorithmic}[1]
    \Require diffusion LLM $\llm{}$, prompt $\vect{p}$, answer length $n$, block length $\blocklen$, diffusion steps $T$, vocabulary $V$, number of blocks $k$.
    \State $\vect{x} \gets \vect{p}$ \Comment{Initialize $\vect{x}$ with input prompt $\vect{p}$}
    \State $\vect{r} \gets \{\}$ \Comment{Intialize the output string}
    \For{$i \in \{1, \dots, k\}$}
    \State $\vect{x}\concat\vect{r_i} \gets \text{Diffusion}(\vect{x}, m + (i-1)\times\blocklen, \blocklen, T, \inalphaLLM{})$ \Comment{$\vect{r_i}\in \inalphaLLM{\blocklen}$ is $i$-th output block}
    \State $\vect{r} \gets \vect{r}\concat\vect{r_i}$
    \State $\vect{x} \gets \vect{x}\concat\vect{r_i}$ \Comment{Compute the input prompt for the next block} 
    \EndFor
    \State \textbf{Return} $\vect{r}$             
    \end{algorithmic}
\end{algorithm}

\begin{algorithm}[H]
    \caption{Semi-Autoregressive Constrained diffusion LLM Generation}
    \label{alg:autoConstraint}
    \begin{algorithmic}[1]
    \Require diffusion LLM $\llm{}$, prompt $\vect{p}$, answer length $n$, block length $\blocklen$, diffusion steps $T$, vocabulary $V$, number of blocks $k$, regular expression $\regex$.
    \State $\dfastart, \live{\dfastates}, \transitions{} \gets \text{PreProcess}
    (\regex)$ \Comment{Pre-compute the dfa start state, live states and $\transitions{}$}
    \State $\vect{x} \gets \vect{p}$ 
    \Comment{Initialize $\vect{x}$ with input prompt $\vect{p}$}
    \State $\vect{r} \gets \{\}$ \Comment{Intialize the output string}
    \State $q_{curr} \gets \dfastart$ \Comment{Intialize the current dfa state the response is at}
    \For{$i \in \{1, \dots, k\}$}
    \State $\vect{x}\concat\vect{r_i}, q_{next} \gets \text{Diffusion}(\vect{x}, m + (i-1)\times\blocklen, \blocklen, T, \inalphaLLM{}, \live{Q}, \transitions{}, q_{curr})$ 
    \If{$q_{next}\not\in \live{\dfastates}$}
     \State \Return None \Comment{No valid completion}
    \EndIf
    \State $\vect{r} \gets \vect{r}\concat\vect{r_i}$
    \State $\vect{x} \gets \vect{x}\concat\vect{r_i}$ 
    \Comment{Compute the input prompt for the next block} 
    \State $q_{curr} \gets q_{next}$ \Comment{Update current DFA state for next block}
    \EndFor
    \State \textbf{Return} $\vect{r}$         
    \end{algorithmic}
\end{algorithm}

In the semi-autoregressive setting, after each block, we ensure that the output generated so far ends in a live state from $\live{\dfastates}$; otherwise, we return the None string (line 7, Algo.~\ref{alg:autoConstraint}). Additionally, we maintain a variable $q_{\text{curr}}$ to track the current DFA state at the end of each block. This state is then used as the starting state for the dynamic programming step in the constrained generation of the next block.
\section{Token Transitions Statistics}
\label{sec:dfa_stats}

\begin{table}[H]
\centering

\caption{Token Transitions Pre-Computation Statistics} 
\begin{tabular}
    {@{}l|c| r r r r@{}}
    \toprule
    & & \multicolumn{2}{c}{GSM-Symbolic} & \multicolumn{2}{c}{JSON-Mode}\\
    Model Family & $|V|$ & Time(s) & \#States & Time(s) & \#States \\
    \midrule
    LLaDA-8B & 126349  & 32.09 & 40 & 13.22 & 169.31 \\
    Dream-7B  & 151667 & 37.01 & 40 & 11.87 & 169.31 \\
    \bottomrule
\end{tabular}

\label{table:token_trans_stats}
\end{table}
In Table~\ref{table:token_trans_stats}, we report the precomputation time and the number of states in the DFA for both tasks. For JSON generation, different regular expressions are used for different schemas; therefore, we report the mean precomputation time and mean number of states. The maximum number of states and precomputation times across all questions are $455$ and $17.7$ (Dream) $21.3$ (LLaDA) seconds, respectively.

\section{GSM-Symbolic}
\label{sec:gsm_info}

\subsection{GSM-Symbolic Prompt}
\label{sec:gsm_prompts}
\lstdefinestyle{myGrammarStyle}{
    basicstyle=\scriptsize\ttfamily, 
    commentstyle=\color{green},
    keywordstyle=\color{blue},
    stringstyle=\color{orange},
    numbers=none, 
    numberstyle=\tiny\color{gray}, 
    breaklines=true, 
    frame=single, 
    framesep=3pt, 
    xleftmargin=5pt, 
    xrightmargin=5pt, 
    backgroundcolor=\color{yellow!0}, 
    tabsize=2, 
    captionpos=b, 
    aboveskip=5pt, 
    belowskip=5pt, 
    linewidth=0.9\linewidth, 
    escapeinside={(*@}{@*)}, 
}

\begin{lstlisting}[style=myGrammarStyle, caption= Prompt template for the GSM-Symbolic task~\cite{mirzadeh2024gsmsymbolicunderstandinglimitationsmathematical}.]
You are an expert in solving grade school math tasks. You will be presented with a grade-school math word problem with symbolic variables and be asked to solve it.

Before answering you should reason about the problem (using the <reasoning> field in the response described below). Intermediate symbolic expressions generated during reasoning should be wrapped in << >>.

Only output the symbolic expression wrapped in << >> that answers the question. The expression must use numbers as well as the variables defined in the question. You are only allowed to use the following operations: +, -, /, //, %, *, and **.

You will always respond in the format described below: 
Let's think step by step. <reasoning> The final answer is <<symbolic expression>>

There are {t} trees in the {g}. {g} workers will plant trees in the {g} today. After they are done, there will be {tf} trees. How many trees did the {g} workers plant today?

Let's think step by step. Initially, there are {t} trees. After planting, there are {tf} trees. The number of trees planted is <<tf - t>>. The final answer is <<tf - t>>.

If there are {c} cars in the parking lot and {nc} more cars arrive, how many cars are in the parking lot?

Let's think step by step. Initially, there are {c} cars. {nc} more cars arrive, so the total becomes <<c + nc>>. The final answer is <<c + nc>>.

{p1} had {ch1} {o1} and {p2} had {ch2} {o1}. If they ate {a} {o1}, how many pieces do they have left in total?

Let's think step by step. Initially, {p1} had {ch1} {o1}, and {p2} had {ch2} {o1}, making a total of <<ch1 + ch2>>. After eating {a} {o1}, the remaining total is <<ch1 + ch2 - a>>. The final answer is <<ch1 + ch2 - a>>.

{p1} had {l1} {o1}. {p1} gave {g} {o1} to {p2}. How many {o1} does {p1} have left?

Let's think step by step. {p1} started with {l1} {o1}. After giving {g} {o1} to {p2}, {p1} has <<l1 - g>> {o1} left. The final answer is <<l1 - g>>.

{question}

\end{lstlisting}
\label{prompt:gsm_prompt}

\subsection{GSM-Symbolic Regex}
\label{sec:gsm_regex}
\lstdefinestyle{myGrammarStyle}{
    basicstyle=\scriptsize\ttfamily, 
    commentstyle=\color{gray},
    keywordstyle=\color{blue},
    stringstyle=\color{orange},
    numbers=none, 
    numberstyle=\tiny\color{gray}, 
    breaklines=true, 
    frame=single, 
    framesep=3pt, 
    xleftmargin=5pt, 
    xrightmargin=5pt, 
    backgroundcolor=\color{yellow!4}, 
    tabsize=2, 
    captionpos=b, 
    aboveskip=5pt, 
    belowskip=5pt, 
    linewidth=0.9\linewidth, 
    escapeinside={(*@}{@*)}, 
}

\begin{lstlisting}[style=myGrammarStyle, caption=GSM-Symbolic Regex]
(?:(?:(?:(?:(?:[ -;=?-~\n]+))*(?:<<(?:(?:\ ))?(?:(?:(?:(?:(?:[a-j])|(?:[0-9]{1,3})|\((?:(?:(?:(?:[a-j])|(?:[0-9]{1,3})|\((?:(?:(?:(?:[a-j])|(?:[0-9]{1,3})|\((?:(?:(?:(?:[a-j])|(?:[0-9]{1,3})))(?:(?:(?:(?:\ ))?(?:(?:\+|\-|//|/|%|\*|\*\*))(?:(?:\ ))?(?:(?:(?:[a-j])|(?:[0-9]{1,3})))))*)\)))(?:(?:(?:(?:\ ))?(?:(?:\+|\-|//|/|%|\*|\*\*))(?:(?:\ ))?(?:(?:(?:[a-j])|(?:[0-9]{1,3})|\((?:(?:(?:(?:[a-j])|(?:[0-9]{1,3})))(?:(?:(?:(?:\ ))?(?:(?:\+|\-|//|/|%|\*|\*\*))(?:(?:\ ))?(?:(?:(?:[a-j])|(?:[0-9]{1,3})))))*)\)))))*)\)))(?:(?:(?:(?:\ ))?(?:(?:\+|\-|//|/|%|\*|\*\*))(?:(?:\ ))?(?:(?:(?:[a-j])|(?:[0-9]{1,3})|\((?:(?:(?:(?:[a-j])|(?:[0-9]{1,3})|\((?:(?:(?:(?:[a-j])|(?:[0-9]{1,3})))(?:(?:(?:(?:\ ))?(?:(?:\+|\-|//|/|%|\*|\*\*))(?:(?:\ ))?(?:(?:(?:[a-j])|(?:[0-9]{1,3})))))*)\)))(?:(?:(?:(?:\ ))?(?:(?:\+|\-|//|/|%|\*|\*\*))(?:(?:\ ))?(?:(?:(?:[a-j])|(?:[0-9]{1,3})|\((?:(?:(?:(?:[a-j])|(?:[0-9]{1,3})))(?:(?:(?:(?:\ ))?(?:(?:\+|\-|//|/|%|\*|\*\*))(?:(?:\ ))?(?:(?:(?:[a-j])|(?:[0-9]{1,3})))))*)\)))))*)\)))))*)\)))(?:(?:(?:(?:\ ))?(?:(?:\+|\-|//|/|%|\*|\*\*))(?:(?:\ ))?(?:(?:(?:[a-j])|(?:[0-9]{1,3})|\((?:(?:(?:(?:[a-j])|(?:[0-9]{1,3})|\((?:(?:(?:(?:[a-j])|(?:[0-9]{1,3})|\((?:(?:(?:(?:[a-j])|(?:[0-9]{1,3})))(?:(?:(?:(?:\ ))?(?:(?:\+|\-|//|/|%|\*|\*\*))(?:(?:\ ))?(?:(?:(?:[a-j])|(?:[0-9]{1,3})))))*)\)))(?:(?:(?:(?:\ ))?(?:(?:\+|\-|//|/|%|\*|\*\*))(?:(?:\ ))?(?:(?:(?:[a-j])|(?:[0-9]{1,3})|\((?:(?:(?:(?:[a-j])|(?:[0-9]{1,3})))(?:(?:(?:(?:\ ))?(?:(?:\+|\-|//|/|%|\*|\*\*))(?:(?:\ ))?(?:(?:(?:[a-j])|(?:[0-9]{1,3})))))*)\)))))*)\)))(?:(?:(?:(?:\ ))?(?:(?:\+|\-|//|/|%|\*|\*\*))(?:(?:\ ))?(?:(?:(?:[a-j])|(?:[0-9]{1,3})|\((?:(?:(?:(?:[a-j])|(?:[0-9]{1,3})|\((?:(?:(?:(?:[a-j])|(?:[0-9]{1,3})))(?:(?:(?:(?:\ ))?(?:(?:\+|\-|//|/|%|\*|\*\*))(?:(?:\ ))?(?:(?:(?:[a-j])|(?:[0-9]{1,3})))))*)\)))(?:(?:(?:(?:\ ))?(?:(?:\+|\-|//|/|%|\*|\*\*))(?:(?:\ ))?(?:(?:(?:[a-j])|(?:[0-9]{1,3})|\((?:(?:(?:(?:[a-j])|(?:[0-9]{1,3})))(?:(?:(?:(?:\ ))?(?:(?:\+|\-|//|/|%|\*|\*\*))(?:(?:\ ))?(?:(?:(?:[a-j])|(?:[0-9]{1,3})))))*)\)))))*)\)))))*)\)))))*))(?:(?:\ ))?>>)))+(?:(?:\.))?)
\end{lstlisting}
\label{gram:gsm_grammar}

\subsection{GSM-Symbolic Case Studies}
\label{sec:gsm_case_studies}
\textbf{Case Study 1:}

\begin{figure}[H]
    \centering
    \includegraphics[width=\linewidth]{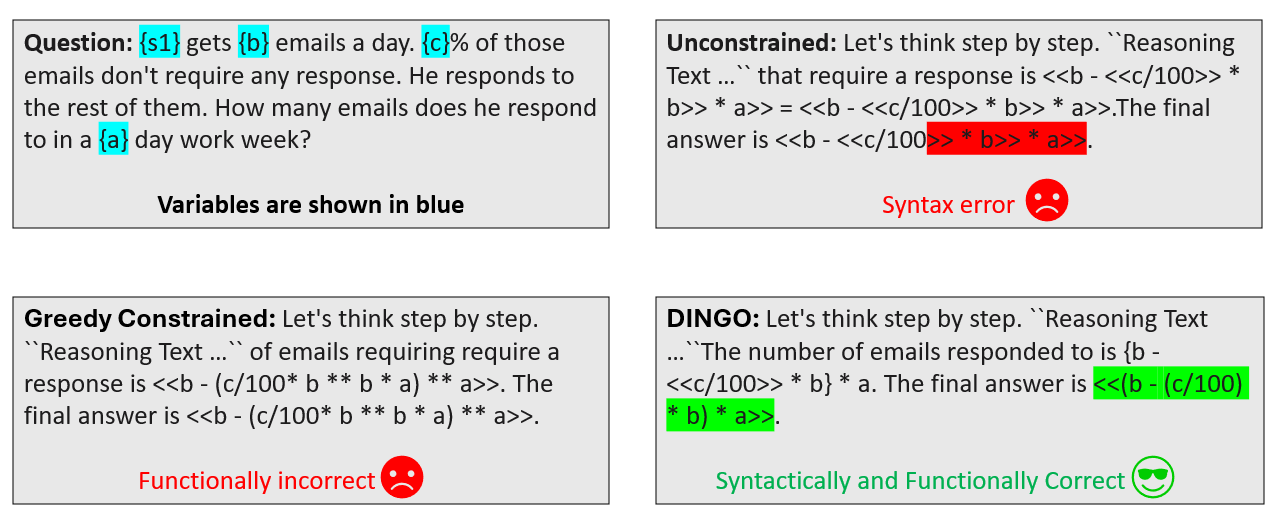}
    \caption{An example from the GSM-symbolic dataset (variables in blue), where unconstrained generation produces syntactically incorrect output, and greedy constrained generation yields a syntactically valid but incorrect answer. In contrast, \Tool{} generates the correct answer.}
    \label{fig:gsm_case_study1}
\end{figure}

\textbf{Case Study 2:}
\begin{figure}[H]
    \centering
    \includegraphics[width=\linewidth]{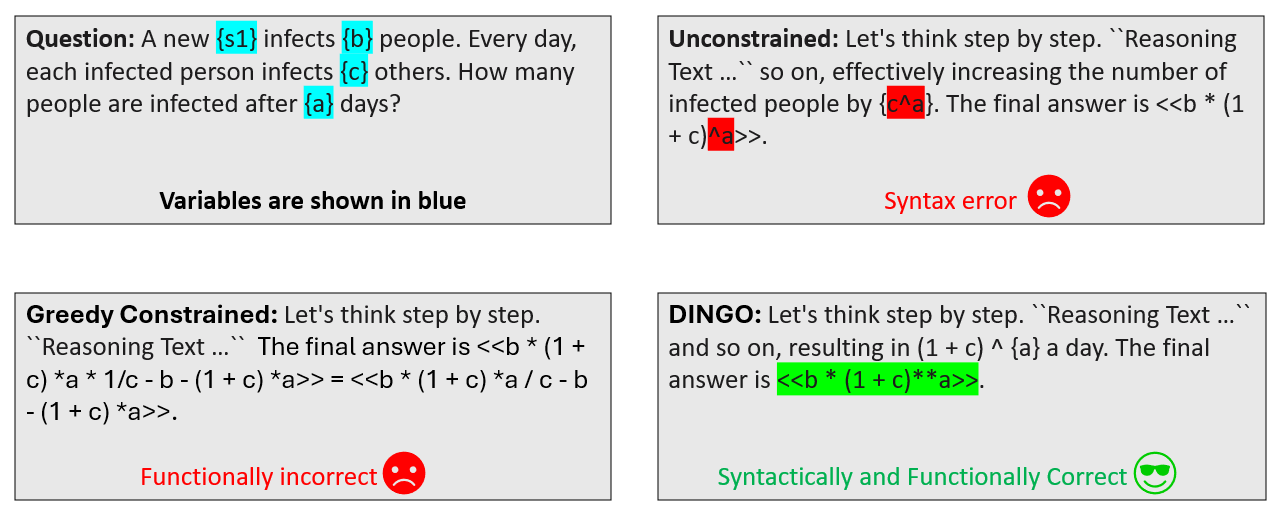}
    \caption{An example from the GSM-symbolic dataset (variables in blue), where unconstrained generation produces syntactically incorrect output, and greedy constrained generation yields a syntactically valid but incorrect answer. In contrast, \Tool{} generates the correct answer.}
    \label{fig:gsm_case_study2}
\end{figure}

\section{JSON-Mode }
\label{sec:jsonmode_info}

\subsection{JSON-Mode Example Prompt}
\label{sec:jsonmode_prompts}
\lstdefinestyle{myGrammarStyle}{
    basicstyle=\scriptsize\ttfamily, 
    commentstyle=\color{green},
    keywordstyle=\color{blue},
    stringstyle=\color{orange},
    numbers=none, 
    numberstyle=\tiny\color{gray}, 
    breaklines=true, 
    frame=single, 
    framesep=3pt, 
    xleftmargin=5pt, 
    xrightmargin=5pt, 
    backgroundcolor=\color{yellow!0}, 
    tabsize=2, 
    captionpos=b, 
    aboveskip=5pt, 
    belowskip=5pt, 
    linewidth=0.9\linewidth, 
    escapeinside={(*@}{@*)}, 
}

\begin{lstlisting}[style=myGrammarStyle, caption= Example JSON Prompt from the JSON-Mode-Eval task~\cite{jsoneval}. The prompt includes a system message that specifies a schema and a user message that explicitly instructs the model to output a JSON object following that schema with certain parameters. ]
You are a helpful assistant that answers in JSON. Here's the json schema you must adhere to:
<schema>
{'title': 'PromotionalCampaign', 'type': 'object', 'properties': {'campaignID': {'title': 'Campaign ID', 'type': 'string'}, 'productID': {'title': 'Product ID', 'type': 'string'}, 'startDate': {'title': 'Start Date', 'type': 'string', 'format': 'date'}, 'endDate': {'title': 'End Date', 'type': 'string', 'format': 'date'}, 'discountDetails': {'title': 'Discount Details', 'type': 'string'}}, 'required': ['campaignID', 'productID', 'startDate', 'endDate']}
</schema>

I'm organizing a promotional campaign for our new eco-friendly laundry detergent, which is part of our household products line. The campaign will start on June 1, 2023, and end on June 30, 2023. We're planning to offer a 15% discount on all purchases during this period. The campaign ID is CAMP123456, the product ID is PROD7891011, and the discount details are 15% off on all purchases.
Only output the JSON object, no other text or comments.
\end{lstlisting}
\label{prompt:json_prompt}

\subsection{JSON-Mode Example Regex}
\label{sec:json_regex}
\lstdefinestyle{myGrammarStyle}{
    basicstyle=\scriptsize\ttfamily, 
    commentstyle=\color{gray},
    keywordstyle=\color{blue},
    stringstyle=\color{orange},
    numbers=none, 
    numberstyle=\tiny\color{gray}, 
    breaklines=true, 
    frame=single, 
    framesep=3pt, 
    xleftmargin=5pt, 
    xrightmargin=5pt, 
    backgroundcolor=\color{yellow!4}, 
    tabsize=2, 
    captionpos=b, 
    aboveskip=5pt, 
    belowskip=5pt, 
    linewidth=0.9\linewidth, 
    escapeinside={(*@}{@*)}, 
}

\begin{lstlisting}[style=myGrammarStyle, caption=Regex for the JSON Schema in Appendix ~\ref{sec:json_regex} ]
\\{[ ]?"campaignID"[ ]?:[ ]?"([^"\\\\\\x00-\\x1F\\x7F-\\x9F]|\\\\["\\\\])*"[ ]?,[ ]?"productID"[ ]?:[ ]?"([^"\\\\\\x00-\\x1F\\x7F-\\x9F]|\\\\["\\\\])*"[ ]?,[ ]?"startDate"[ ]?:[ ]?"(?:\\d{4})-(?:0[1-9]|1[0-2])-(?:0[1-9]|[1-2][0-9]|3[0-1])"[ ]?,[ ]?"endDate"[ ]?:[ ]?"(?:\\d{4})-(?:0[1-9]|1[0-2])-(?:0[1-9]|[1-2][0-9]|3[0-1])"([ ]?,[ ]?"discountDetails"[ ]?:[ ]?"([^"\\\\\\x00-\\x1F\\x7F-\\x9F]|\\\\["\\\\])*")?[ ]?\\}
\end{lstlisting}
\label{gram:json_schema_grammar}

\subsection{JSON-Mode Case Studies}
\label{sec:jsonmode_case_studies}
\begin{figure}[H]
    \centering
    \includegraphics[width=\linewidth]{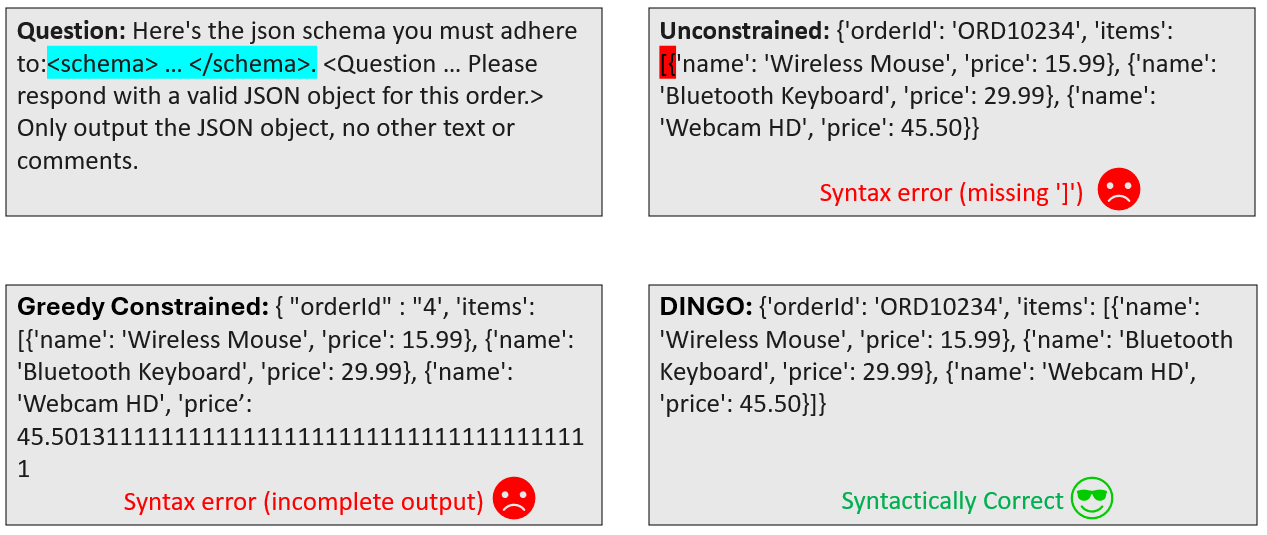}
    \caption{An example from JSON generation, where unconstrained generation produces a syntactically incorrect output, and greedy constrained generation yields a valid but incomplete prefix. In contrast, \Tool{} generates a syntactically correct answer.}
    \label{fig:json_case_study1}
\end{figure}
\label{gram:json_schema_case_study1}
\begin{figure}[H]
    \centering
    \includegraphics[width=\linewidth]{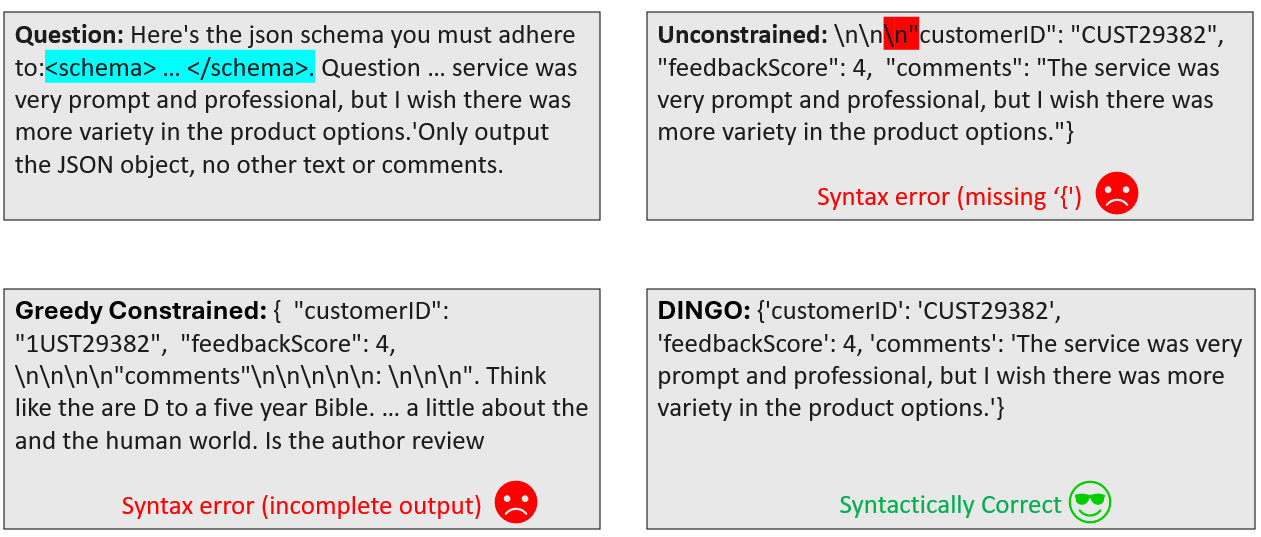}
    \caption{An example from JSON generation, where unconstrained generation produces a syntactically incorrect output, and greedy constrained generation yields a valid but incomplete prefix. In contrast, \Tool{} generates a syntactically correct answer.}
    \label{fig:json_case_study2}
\end{figure}
\label{gram:json_schema_case_study2}



\clearpage
\newpage

\section{Ablation Study on Number of Blocks for Diffusion LLM Generation (GSM-Symbolic)}
\label{sec:block_abl_gsm}
We run generation with a response length of 128, using 64 total diffusion steps, and each of 1, 2, and 8 blocks. Table ~\ref{tab:gsm_symbolic_comparison_blocks_ablation} presents the result. 

\begin{table*}[h]
    \centering
    \small
    \caption{Ablation Study on The Number of Diffusion Blocks for GSM-Symbolic}
    \begin{tabular}{@{}lllrrr@{}}
        \toprule
        \textbf{Model} & \textbf{\#Blocks} & \textbf{Method}  & \textbf{Acc. (\%)} & \textbf{Parse (\%)} &  \textbf{Time (s)} \\
        
\midrule
  & & \unconstrained{} & 20 & 54 & 23.66\\
& & \greedyconstrained{} & 26 & 94 & 23.7\\
& 1 & \bestofbaseline{} & 26 & 94 & 23.66\\
 & & \Tool{} & 29 & 100 & 23.73\\

 & & \unconstrained{} & 22 & 54 & 23.63\\
 & & \greedyconstrained{} & 30 & 96 & 23.81\\
 LLaDA-8B-I & 2 & \bestofbaseline{} & 30 & 96 & 23.65\\
 & & \Tool{} & 32 & 100 & 23.93\\

& & \unconstrained{} & 19 & 35 &  23.78\\
& & \greedyconstrained{} & 27 & 98 & 23.97 \\
 & 8 & \bestofbaseline{} & 27 & 98 & 23.8 \\
 & & \Tool{} & \textbf{32} & \textbf{100} & 23.92\\

\midrule
 & & \unconstrained{} & 28 & 69 & 23.56\\
 & & \greedyconstrained{} & 32 & 90 & 23.64\\
 & 1 & \bestofbaseline{} & 32 & 90 & 23.65\\
 & & \Tool{} & 34 & 100 & 23.67\\
 
 & & \unconstrained{} & 30 & 55 & 23.62\\
 & & \greedyconstrained{} & 33 & 87 & 23.71\\
 Dream-I-7B & 2 & \bestofbaseline{} & 33 & 87 & 23.62\\
 & & \Tool{} & 34 & 100 & 23.65\\

& & \unconstrained{} & 32 & 61 & 23.89\\
 & & \greedyconstrained{} & 34 & 93 & 24.01 \\
 & 8 & \bestofbaseline{} & 34 & 93 & 23.89 \\
 & & \Tool{} & \textbf{36} & \textbf{100} & 23.91 \\
 
\bottomrule
    \end{tabular}
    \label{tab:gsm_symbolic_comparison_blocks_ablation}
\end{table*}
\clearpage
\newpage

\section{Ablation Study on Number of Blocks for Diffusion LLM Generation (JSON-Mode)}
\label{sec:block_abl}
We run generation with a response length of 128, using 64 total diffusion steps, and each of 1, 2, and 8 blocks. Table ~\ref{tab:json_symbolic_comparison_blocks_ablation} presents the result. 

\begin{table*}[h]
    \centering
    \small
    \caption{Ablation Study on The Number of Diffusion Blocks for JSON-Mode.}
    \begin{tabular}{@{}lllrrr@{}}
        \toprule
        \textbf{Model} & \textbf{\#Blocks} & \textbf{Method}  & \textbf{Acc. (\%)} & \textbf{Parse (\%)} &  \textbf{Time (s)} \\
        
\midrule
  & & \unconstrained{} & 87 & 91 & 6.7\\
& & \greedyconstrained{} & 78 & 79 & 6.81\\
& 1 & \bestofbaseline{} & 99 & 99 & 6.73\\
 & & \Tool{} & 100 & 100 & 6.78\\

 & & \unconstrained{} & 84 & 92 & 6.72\\
 & & \greedyconstrained{} & 92 & 94 & 6.83\\
 LLaDA-8B-I & 2 & \bestofbaseline{} & 99 & 99 & 6.73\\
 & & \Tool{} & 100 & 100 & 6.86\\

& & \unconstrained{} & 84 & 89 &  6.73\\
& & \greedyconstrained{} & 98 & 98 & 6.87 \\
 & 8 & \bestofbaseline{} & \textbf{100} & \textbf{100} & 6.75 \\
 & & \Tool{} & 100 & 100 & 6.85\\

\midrule
 & & \unconstrained{} & 85 & 87 & 6.4\\
 & & \greedyconstrained{} & 30 & 30 & 6.51\\
 & 1 & \bestofbaseline{} & 91 & 93 & 6.43\\
 & & \Tool{} & \textbf{100} & \textbf{100} & 6.55\\
 
 & & \unconstrained{} & 79 & 82 & 6.47\\
 & & \greedyconstrained{} & 37 & 39 & 6.68\\
 Dream-I-7B & 2 & \bestofbaseline{} & 86 & 88 & 6.5\\
 & & \Tool{} & 100 & 100 & 6.63\\

& & \unconstrained{} & 70 & 74 & 6.44\\
 & & \greedyconstrained{} & 52 & 52 & 6.65 \\
 & 8 & \bestofbaseline{} & 86 & 89 & 6.46 \\
 & & \Tool{} & 100 & 100 & 6.67 \\
 
\bottomrule
    \end{tabular}
    \label{tab:json_symbolic_comparison_blocks_ablation}
\end{table*}
\clearpage
\newpage

\section{Ablation Study on Number of Steps for Diffusion LLM Generation (GSM-Symbolic)}
\label{sec:steps_abl_gsm}
We run generation with a response length of 128, 1 block, and each of 16, 32, 64, and 128 total diffusion steps. Table ~\ref{tab:gsm_symbolic_comparison_steps_ablation} presents the result. 

\begin{table*}[h]
    \centering
    \small
    \caption{Ablation Study on The Number of Diffusion Steps for GSM-Symbolic with Dream-I-7B}
    \begin{tabular}{@{}llrrr@{}}
        \toprule
        \textbf{\#Steps} & \textbf{Method}  & \textbf{Acc. (\%)} & \textbf{Parse (\%)} &  \textbf{Time (s)} \\
        
\midrule
  & \unconstrained{} & 6 & 20 & 5.99\\
 & \greedyconstrained{} & 13 & 78 & 6.18\\
 16 & \bestofbaseline{} & 13 & 78 & 5.99\\
 & \Tool{} & \textbf{18} & \textbf{100} & 6.09\\

 \midrule
 & \unconstrained{} & 18 & 48 & 11.96\\
 & \greedyconstrained{} & 25 & 87 & 12.06\\
 32 & \bestofbaseline{} & 25 & 87 & 11.96\\
 & \Tool{} & \textbf{28} & \textbf{100} & 12.03\\
\midrule
 & \unconstrained{} & 28 & 69 & 23.56\\
 & \greedyconstrained{} & 32 & 90 & 23.64\\
 64 & \bestofbaseline{} & 32 & 90 & 23.65\\
 & \Tool{} & \textbf{34} & \textbf{100} & 23.67\\

\midrule
 & \unconstrained{} & 31 & 74 & 47.83\\
 & \greedyconstrained{} & 30 & 89 & 47.88 \\
 128 & \bestofbaseline{} & 31 & 90 & 47.83 \\
 & \Tool{} & \textbf{33} & \textbf{100} & 47.86 \\
 
\bottomrule
    \end{tabular}
    \label{tab:gsm_symbolic_comparison_steps_ablation}
\end{table*}
\clearpage
\newpage

\section{Ablation Study on Number of Steps for Diffusion LLM Generation (JSON-Mode)}
\label{sec:steps_abl_jsonmode}
We run generation with a response length of 128, 1 block, and each of 16, 32, 64, and 128 total diffusion steps. Table ~\ref{tab:json_mode_comparison_steps_ablation} presents the result. 

\begin{table*}[h]
    \centering
    \small
    \caption{Ablation Study on The Number of Diffusion Steps for JSON-Mode with Dream-I-7B}
    \begin{tabular}{@{}llrrr@{}}
        \toprule
        \textbf{\#Steps} & \textbf{Method}  & \textbf{Acc. (\%)} & \textbf{Parse (\%)} &  \textbf{Time (s)} \\
        
\midrule
  & \unconstrained{} & 54 & 59 & 1.51\\
 & \greedyconstrained{} & 32 & 32 & 1.62\\
 16 & \bestofbaseline{} & 68 & 71 & 1.52\\
 & \Tool{} & \textbf{100} & \textbf{100} & 1.6\\

 \midrule
 & \unconstrained{} & 67 & 71 & 3.23\\
 & \greedyconstrained{} & 35 & 35 & 3.35\\
 32 & \bestofbaseline{} & 78 & 82 & 3.24\\
 & \Tool{} & \textbf{100} & \textbf{100} & 3.31\\

\midrule
 & \unconstrained{} & 85 & 87 & 6.4\\
 & \greedyconstrained{} & 30 & 30 & 6.51\\
 64 & \bestofbaseline{} & 91 & 93 & 6.43\\
 & \Tool{} & \textbf{100} & \textbf{100} & 6.55\\

\midrule
 & \unconstrained{} & 85 & 87 & 13.42\\
 & \greedyconstrained{} & 46 & 46 & 13.53 \\
 128 & \bestofbaseline{} & 95 & 97 & 13.43 \\
 & \Tool{} & \textbf{100} & \textbf{100} & 13.51 \\
 
\bottomrule
    \end{tabular}
    \label{tab:json_mode_comparison_steps_ablation}
\end{table*}
\clearpage
\newpage
    
\end{document}